\documentclass{article}
\usepackage{amssymb, amsfonts, amsthm, amsmath, amsbsy, latexsym, color,bbm,url}
\usepackage{graphics,graphicx}
\usepackage{hyperref} 
\usepackage{mathtools}
\usepackage{caption}
\usepackage{subcaption}
\usepackage[dvipsnames]{xcolor}
\usepackage{fullpage}
\usepackage{pifont}

\newcommand{\bR}{\mathbb{R}}

\newcommand{\bZ}{\mathbb{Z}}

\newcommand{\bL}{\mathbf{L}}
\newcommand{\bC}{\mathbf{C}}
\newcommand{\bW}{\mathbf{W}}

\newcommand{\rar}{\rightarrow}

\newcommand{\ip}[2]{\left<#1,#2\right>}
\newcommand{\norm}[1]{\left\Vert#1\right\Vert}
\newcommand{\abs}[1]{\left\vert#1\right\vert}

\newcommand{\beq}{\begin{equation}}
\newcommand{\eeq}{\end{equation}}
\newcommand{\GL}{\operatorname{GL}}

\newcommand{\SE}{\operatorname{SE}}
\newcommand{\AGL}{\operatorname{AGL}}
\newcommand{\SO}{\operatorname{SO}}
\newcommand{\trace}{\operatorname{tr}}
\newcommand{\dist}{\operatorname{dist}}
\newcommand{\disim}{\operatorname{disim}}

\newcommand{\supp}{\operatorname{supp}}

\newcommand{\ournet}{TraPNet}

\theoremstyle{plain}
\newtheorem{thm}{Theorem}[section]
\newtheorem{cor}[thm]{Corollary}
\newtheorem{prop}[thm]{Proposition}

\newtheorem*{thm*}{Theorem}
\newtheorem*{cor*}{Corollary}
\newtheorem*{prop*}{Proposition}

\theoremstyle{definition}
\newtheorem{defn}[thm]{Definition}
\newtheorem{lem}[thm]{Lemma}
\newtheorem{ex}[thm]{Example}
\newtheorem*{defn*}{Definition}
\newtheorem*{lem*}{Lemma}
\newtheorem*{ex*}{Example}
\newtheorem*{aside*}{Aside}
\newtheorem{remark}[thm]{Remark}

\newtheorem{example}[thm]{Example}
\newtheorem*{nota*}{Notation}

\definecolor{MichiganBlue}{HTML}{00274C}
\definecolor{MichiganMaize}{HTML}{FFCB05}

\title{Transversal Pooling Neural Networks}
\author{
Emily J.\ King\thanks{
Department of Mathematics, Colorado State University,
Fort Collins, CO, USA;
Corresponding author: emily.king@colostate.edu
}
\and
Dustin G.\ Mixon\thanks{
Department of Mathematics and Translational Data Analytics Institute,
The Ohio State University, Columbus, OH, USA
}
\and
Michael Perlmutter\thanks{
Department of Mathematics, Boise State University,
Boise, ID, USA
}
\and
Lander Ver Hoef\thanks{
Cooperative Institute for Research in the Atmosphere,
Colorado State University, Fort Collins, CO, USA
}
}
\date{}

\begin{document}

\maketitle

\begin{abstract}
Many learning tasks require stability to small transformations while retaining sensitivity to larger ones.
We introduce \emph{transversal pooling neural networks}, which generalize spatial max pooling to affine group actions.
We establish equivariance to a chosen subgroup and derive explicit stability bounds for individual pooled wavelet coefficients under affine perturbations of the input. 
Experimentally, we demonstrate the utility of our networks in low-data environments and for predicting tropical cyclone intensification. 
\end{abstract}

\section{Introduction}

A feedforward neural network is a composition of affine linear and nonlinear maps that converts raw input data $f$ into a hidden representation $\Theta(f)$ that facilitates downstream tasks like classification or regression. 
Ideally, this hidden representation captures relevant information while being stable to uninformative perturbations. 
For example, consider the computer vision task of recognizing an object in an image.
Since translating an image does not change the object, it should not change the hidden representation. 
That is, we want the desired hidden representation to be \emph{invariant} to the action of the translation group:
\[
 \Theta( g \cdot f) = \Theta (f),   
\]
where $g \cdot f$ denotes the result of transforming the image $f$ by the translation $g$.

Here, we consider the extension of this property to other groups of affine operators, e.g., groups that contain dilations, shears, or rotations and, when appropriate, we seek for invariance to only hold locally. 
Accordingly, our work is motivated by the following considerations: 
\begin{enumerate}
    \item[(a)] in practice, full group invariance can discard relevant information;
    \item[(b)] the standard max pooling nonlinearity in modern neural networks requires careful adaptation to accommodate general affine transformations; and
    \item[(c)] 
    we need quantitative stability guarantees under affine perturbations.
\end{enumerate}
Here, (a) is a statement about the learning task, while (b) and (c) concern the architecture we wish to use.

As a simple example of (a), consider the task of identifying a handwritten digit~\cite{lecun1998mnist}.
Given the leftmost `6' in  Figure~\ref{fig:affnist}, we take inspiration from affNIST~\cite{tieleman2013affnist} by applying affine transformations close to the identity to generate the middle three `6's; any successful machine learning algorithm should recognize each of these as a `6.'
However, we probably want to classify the rightmost digit, which results from a $180$-degree rotation, as a `9.' 
This suggests that the hidden representation should be stable to \emph{small} rotations, which you might call \emph{approximate invariance}; interestingly, certain biological neurons are known to exhibit such approximate invariance~\cite{logothetis1995shape}.
For an application in atmospheric science, consider the task of predicting whether a tropical cyclone will undergo rapid intensification.
In this setting, storm motion introduces asymmetries that make the orientation informative.
This motivates approximate invariance to small rotations while allowing sensitivity to larger ones.

\begin{figure}[t]
    \centering
    \includegraphics[width=0.75\linewidth]{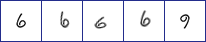}
    \caption{Four affine transformations of the leftmost digit `6.' The middle three images result from transformations that are close to the identity.  The rightmost image results from a $180$-degree rotation. This classification task (recognizing digits) is well suited for the transversal pooling architecture we introduce.}
    \label{fig:affnist}
\end{figure}

In this paper, we address (a) and (b) by introducing \emph{transversal pooling neural networks} (\ournet{}s), a family of architectures that provide approximate invariance to small group transformations while allowing sensitivity to larger ones. 
To address~(c), we derive explicit bounds on changes in individual pooled wavelet coefficients under affine perturbations of the input. 
Finally, we evaluate \ournet{}s on the tropical cyclone task described above, as well as on a synthetic model that was inspired by it. 
The synthetic experiments demonstrate improved performance with limited training data compared with convolutional neural network baselines trained using data augmentation. 
We provide a more detailed overview of our results in Section~\ref{sec: overview and road map}, but first, we discuss the relevant literature.

\subsection{Related Work}
\subsubsection{Convolutional Neural Networks and Pooling}

Many classical signal and image processing techniques---such as Canny edge detection~\cite{canny1986computational}---are powered by the local pattern matching provided by convolution.
In machine learning, these ideas appear in the form of \emph{convolutional neural networks (CNNs)}, which were introduced in~\cite{fukushima1980neocognitron,waibel1989phoneme,lecun1989backpropagation} to exploit translation symmetry.
(See~\cite{li2021survey} for a survey.)
A \emph{convolutional layer} computes the cross-correlation\footnote{In machine learning, the term \emph{convolution} typically refers to the notion of \emph{cross-correlation} in signal processing. We use the latter term to distinguish from the standard mathematical notion of convolution that involves time reversal.} between its input and learned templates to produce feature maps, followed by a pointwise nonlinear \emph{activation function}. 
A common activation function is the rectified linear unit (ReLU), which replaces each negative entry with zero.
In many CNN architectures, one or more convolutional layers are followed by a \emph{(local) pooling} layer.
In pooling layers, a neighborhood of values is replaced with a single value, such as their maximum (known as \emph{max pooling}) or their average.
See the left-hand side of Figure~\ref{fig:max pooling comparison} for an illustration.
There, the input $6\times4$ grid of numbers represents \emph{feature activations} (i.e., the output of a discrete convolution and activation function) and the output $3\times 2$ grid of numbers comes from pooling over $2 \times 2$ grids.
An early use of max pooling appears in a speech-recognition network that selects maximum responses over temporal search windows as part of a time-alignment procedure \cite{yamaguchi1990neural}.

\begin{figure}[t]
    \centering
\includegraphics[width=\linewidth]{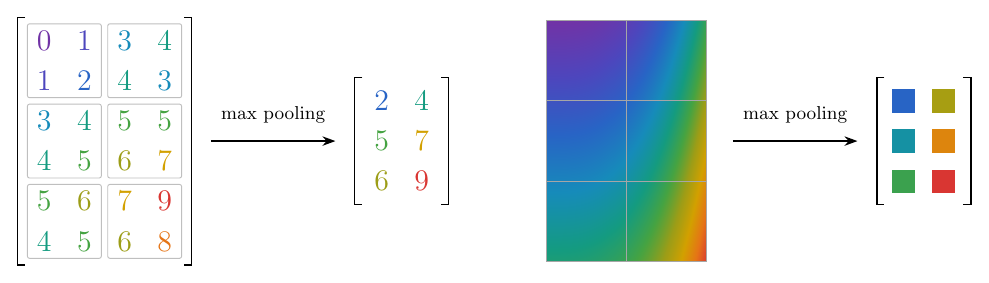}

    \caption{Left: A $6\times 4$ grid of feature activations passes through a max pooling layer that converts it into a $3\times 2$ grid of representative activations. Right: We apply max pooling to a continuous field of activations (specifically, $f(x,y) = 6x^3-14y$ for $x\in [0,2]$ and $y\in [0,3]$) to reduce it to a $3\times 2$ grid of representative activations. This variant of max pooling, which we denote by~$\Xi$, is developed in the more general context of affine groups and transversals in Section~\ref{sec:GEfilterbanks}.}
    \label{fig:max pooling comparison}
\end{figure}

In a conventional CNN for classification, the feature maps produced by successive convolutional and pooling layers are flattened into a vector and passed through affine linear maps (called \emph{fully connected layers}) and activation functions.
One may loosely think of the convolutional layers as extracting local features and the fully connected layers as processing the features to perform the intended task.
Pooling operators shrink the size of the input signal and thereby decrease computational cost by reducing the number of parameters needed in the fully connected layers. 
In addition, they retain coarse spatial information about feature responses while discarding their precise locations within each pooling region.
In some cases, max pooling introduces stability to translation~\cite{leterme2022shift}, but in other cases, it introduces sensitivity to small translations~\cite{zhang2019making}.
Additionally, while max pooling operators were critical in early CNNs trained to classify images, they are only proven to be helpful for certain problems and datasets~\cite{matoba2023benefits}.

\subsubsection{Wavelets and the Scattering Transform}\label{sec:wavelets}

In this paper, we model each input as a real-valued function $f\in\bL^2(\bR^d)$, and we extract features by taking inner products with transformed copies of a real-valued template $\psi\in\bL^2(\bR^d)$.
We transform the template by members of a closed subgroup $G$ of the affine group $\AGL(\bR^d) = \GL(\bR^d) \ltimes \bR^d$ according to the usual quasiregular representation.
Namely, for $g=(A,b) \in \AGL(\bR^d)$,
our transformed templates take the form
\begin{equation}\label{eqn:quasireg}
(g \cdot \psi)(x) 
= \frac{1}{\abs{\det(A)}^{1/2}} \psi(A^{-1} (x-b)).
\end{equation}
Once we fix a template $\psi$ and a group $G$ of transformations, the first layer of our network converts the raw data $f$ into the $G$-indexed collection $\Psi (f): G \rar \bR$ of features defined by
\begin{equation}\label{eqn:Psi}
\Psi(f)(g) = \ip{f}{g \cdot \psi}_{\bL^2(\bR^d)}.
\end{equation}
This map $\Psi$ goes by various names in the wavelet theory literature, including the \emph{continuous wavelet transform} and the \emph{voice transform} (see, e.g.,~\cite{grossmann1985transforms,feichtinger1989banach,heil1989continuous,ali2000coherent,fuehr2005abstract}).
Continuous wavelet transforms (or the discrete subsampling thereof) with preassigned templates (called \emph{wavelets}\footnote{Many authors require a ``wavelet'' to satisfy an admissibility condition; see the brief discussion in Section~\ref{sec:GEfilterbanks}. Here, we simply mean a template acted upon by the quasiregular representation of (some closed subgroup of) the affine group.})  
have proven themselves over the last half century to be a powerful tool in image and signal processing and related applications~\cite{daubechies1992ten,benedetto1993wavelets,unser1996review}.

To emphasize the relationship with convolutional neural networks, we observe that the output of $\Psi$ may indeed be expressed in terms of convolutions.
Fix $\psi \in \bL^2(\bR^d)$ and denote $\tilde{\psi}(x) = \psi(-x)$.
Then for any $f \in \bL^2(\bR^d)$ and $(A,b) \in G$, we have
\begin{align*}
    \Psi(f)(A,b) &= \int_{\bR^d} f(x)  \frac{1}{\abs{\det(A)}^{1/2}} \psi(A^{-1}(x -b)) \, dx\\
    &= \int_{\bR^d}\frac{1}{\abs{\det(A)}^{1/2}}  \tilde{\psi}(A^{-1}y) f(b-y) \, dy \\
     &= f \ast ((A,0) \cdot \tilde{\psi}) (b),
\end{align*}
where $\ast$ denotes the standard Euclidean convolution.
Thus, the output of $\Psi$ is obtained by evaluating a family of ordinary convolutions at the affine shifts that arise in $G$.

Due to their widespread use in image and signal processing, wavelet transforms are natural objects to use in convolution neural network architectures. 
\emph{Scattering transforms} extract features by cascading fixed wavelet convolutions and pointwise modulus operations~\cite{mallat:scattering2012} and then applying an averaging-type operation (typically low-pass filtering). 
They approach full translation invariance as the scale of the averaging increases and also satisfy stability bounds for small smooth deformations.
Furthermore, \cite{mallat:scattering2012} shows that this construction can be extended to have desirable symmetries with respect to groups of rotations, in addition to the translation group.
The construction and analysis of the scattering transform have subsequently been extended to more general classes of filters \cite{oyallon:scalingScattering2017,CZAJA2019149,wiatowski:frameScat2015,wiatowski:mathTheoryCNN2018,eickenberg:scatMoleculesJCP2018,czaja2024translation} as well as to graphs \cite{zou2020graph,gama:diffScatGraphs2018,gao2019geometric}, manifolds \cite{perlmutter2020geometric} and other more general domains \cite{chew2024geometric} as part of the growing field of geometric deep learning \cite{bronstein2017geometric,bronstein2021geometric}. 

One may also integrate wavelet transformations into trained neural networks.
Some approaches include creating learned neural network components which emulate multiresolution analysis~\cite{rodriguez2020deep}, using trained wavelets to pool~\cite{wolter2021adaptive}, mixing (untrained) scattering layers and trained CNN layers \cite{cotter2019learnable}, and parameterizing scattering-like neural network architectures to allow modification of the wavelet transforms based on the data~\cite{gauthier2022parametric}.

We note that our theoretical analysis takes place in  $\bL^2(\bR^d)$, whereas in practice, neural networks are applied to finite-dimensional data (e.g., images with finitely many pixels).
Accordingly, we conduct numerical experiments in Section~\ref{sec: continuity of pooling} that demonstrate how certain ideas in our theory transfer to the finite-dimensional setting. 
This follows a well-established tradition in wavelet theory in which results for functions in $\bL^2(\bR^d)$ help explain the behavior of digital implementations.
For example, the structure of \emph{multiresolution analyses} over $\bL^2(\bR)$ is intimately tied to filter banks for finite data~\cite{mallat1989multiresolution}.
Additionally, suitable shearlet systems achieve near-optimal approximation rates for \emph{cartoon-like functions}~\cite{guo2007optimally,kutyniok2011compactly}, which helps explain why shearlets perform well on digital images with similar structure.

\subsubsection{Invariant and Equivariant Learning}

Wavelet-inspired neural network architectures are successful in part because they exploit symmetries.
Before explaining this, we first distinguish invariance, equivariance, and group-Lipschitz continuity.

\begin{defn}
Let $G$ be a group acting on a set $X$. 
We say that a function $f:X\to Y$ is
\begin{itemize}
    \item[(a)] \emph{invariant (to $G$)} if $f(g \cdot x) = f(x)$ for all $x \in X$, $g \in G$,
    \item[(b)] \emph{equivariant\footnote{
    Some authors use ``covariance'' for what we call ``equivariance''; see, for example,~\cite{sifre2013rotation}.} (with respect to $G$)} if $G$ also acts on $Y$ and $f(g \cdot x) = g \cdot f(x)$ for all $x \in X$, $g \in G$, and
    \item[(c)] \emph{group-Lipschitz (with respect to $G$)} if $G$ is a metric group with metric $d$ (or more generally a semimetric), $Y$ is a normed vector space, and for all $x \in X$ there exists $C_x > 0$ such that  $\norm{f(g \cdot x)-f(h\cdot x)} \leq C_x d(g,h)$ for all $g, h \in G$.
\end{itemize}
\end{defn}
In other words, if $f$ is invariant, then perturbing the input by any $g \in G$ does not change the output; if $f$ is equivariant, then modifying the input by $g$ affects the output in a predictable manner; and if $f$ is group-Lipschitz, then it is stable to transformations in $G$ that are close to the identity.
As a simple example of equivariance, consider the task of computing centroids of the digits in Figure~\ref{fig:affnist}. 
Once we know the centroid of the leftmost `6' and the transformation to the second `6,' then we may similarly transform the centroid to obtain the centroid of the transformed `6.'

There are many approaches in machine learning to incorporate invariance and equivariance, including enforcing invariance in a neural network architecture~(e.g., \cite{fasel2006rotation,dieleman2015classifying,cohen2016group,marcos2016learning,dieleman2016exploiting,laptev2016ti}), data augmentation~\cite{simard2003best}, steerable filters~\cite{weiler2019general,weiler20183d}, averaging outputs of networks to pick up features at each given $g$~\cite{ciregan2012multi,wu2015deep}, tangent propagation~\cite{simard1991tangent}, mapping inputs to ``canonical'' orientations~\cite{jaderberg2015spatial}, incorporating priors into restricted Boltzmann machines~\cite{schmidt2012learning,kivinen2011transformation,sohn2012learning}, and max filter banks~\cite{cahill2025group,mixon2026injectivity,mixon2023max}. Additionally, equivariant graph neural networks \cite{satorras2021n,batzner20223,han2022geometrically,duval2023hitchhiker,johnson2026vdwgnns} leverage natural symmetries with respect to  rigid motions and vertex permutations to improve performance on, e.g., molecule-related tasks (where each molecule is modeled as a geometric graph with vertices in  $\mathbb{R}^3$).
See also~\cite{della2019deep} and the references therein for further techniques to obtain equivariance to rigid body motions in machine learning models.

\subsubsection{The Geometric Deep Learning Blueprint}

Our transversal pooling architecture builds on networks that combine equivariant feature extraction with invariant pooling~ (e.g.,~\cite{fasel2006rotation,dieleman2015rotation,dieleman2015classifying,cohen2016group,marcos2016learning,dieleman2016exploiting,laptev2016ti,worrall2018cubenet,bekkers2018roto,romero2020wavelet,hong2022multiscale}).
This construction follows the \emph{geometric deep learning blueprint}~\cite{bronstein2021geometric}.
To obtain a network that is invariant to a closed subgroup $L \leq \GL(\bR^d)$, one first performs the wavelet transformation~$\Psi$ as described in~\eqref{eqn:Psi} with learned templates $\psi$ and with $G=L \ltimes \bR^d$.
This operation, often called a \emph{lifting convolution}~\cite{romero2020wavelet,bekkers2018roto}, maps functions on $\mathbb R^d$ to functions on $G$.
See the leftmost portion of Figure~\ref{fig:\ournet{} example} for an illustration.
Here, $f$ is the indicator function of a shape and $L \cong C_{12}$ is the group of rotations by multiples of 30~degrees.
Note that the output of $\Psi$ in this case is a data cube parameterized
by $L \ltimes \bR^d$, and furthermore, if we rotate the input by a multiple of 30~degrees, then the slices of the data cube are cyclically permuted, i.e., $\Psi$ is equivariant to $L \cong C_{12}$. 
After passing through an equivariant layer of the form $\Psi$, 
an architecture might perform \emph{group convolution} (as in~\eqref{eqn: upsilon}) in subsequent layers to preserve the equivariance before ``going invariant.''
Group convolution is closely tied to the requirement of equivariance: for feature spaces indexed by homogeneous spaces of a compact group, \cite{kondor2018generalization}~characterizes equivariant linear layers as generalized group convolutions. 
Equivariance can also be maintained by applying ReLU activations; \cite{gibson2024equivariant}~develops a piecewise linear representation theory to analyze how these nonlinearities interact with the decomposition of layers into irreducible representations.
To finally achieve invariance, one might perform $L$-invariant pooling by taking the pixelwise maximum across all slices in the data cube output of~$\Psi$.
Indeed, the maximum of a list is invariant to permutation, and more generally, any permutation-invariant pooling method would yield an $L$-invariant network; testing in \cite{hong2022multiscale} showed a marginal improvement from taking the maximum instead of the average.
Interestingly, $L$-invariant pooling also appears earlier in the literature~\cite{riesenhuber1999hierarchical} in the context of translation-scale-invariant complex cells in the cortex.

Why apply multiple equivariant layers before going invariant?
Because \emph{premature invariance} frequently degrades performance \cite{cohen2016group,coors2018learning,ver2023using}.
We note that two early versions of $L$-invariant networks used different strategies to combine features from different orientations: In~\cite{fasel2006rotation} the results are ``blurred'' over smaller subsets of rotations after each layer until the final pooling collapses all rotations into one element, and in~\cite{dieleman2015rotation} the feature maps under different rotations and reflections are concatenated before passing through fully connected layers.
We also note that~\cite{cohen2016group} introduces a general pooling framework that takes the maximum of a feature map over each left translate $gU$ of a fixed neighborhood $U$ of the identity, then subsamples by retaining the pooled values only at elements $g$ of a chosen subgroup of $G$.
This very general setup includes the transversal pooling of this paper as an example; however, they then specialize this general pooling structure to the setting where $U$ is a \emph{subgroup} of $G$ and subsample over a set of coset representatives.
This approach is in some sense orthogonal both to our setup and to the translation-based max pooling used in conventional CNNs.

\subsubsection{Symmetric Architectures versus Symmetric Training Sets}

The previous section follows the geometric deep learning blueprint~\cite{bronstein2021geometric} to construct a symmetric architecture.
As an alternative, one might use ideas from \emph{data augmentation} to compile a symmetric training set.
Here, we enlarge the training set to include examples that were transformed in ways that preserve their labels. 
(See, e.g.,~\cite{simard2003best} and references therein.)
For example, instead of only using the leftmost `6' in Figure~\ref{fig:affnist} to train a network, we might add the middle three `6's to the training set.
By planting this information in the training set, we can expect a vanilla neural network to learn the desired symmetries instead of carefully designing a symmetric architecture to accomplish the same task.
Not only is data augmentation easy to implement, it also allows one to achieve approximate invariance to modifications that are not readily modeled by affine group actions, such as nonrigid transformations like elastic deformations~\cite{simard2003best} or noninvertible transformations like cropping.
Still, learning symmetry through data augmentation can present some challenges.
Indeed, in studying generalization from partially observed group orbits, \cite{perin2025ability}~finds that performance suffers when architectures do not explicitly encode the relevant symmetry.

Furthermore, a symmetric architecture comes with a number of advantages:
\begin{enumerate}
  \item A symmetric architecture restricts the optimization search space, which can accelerate convergence and reduce the amount of training data required.
  \item By virtue of weight sharing, a symmetric architecture can carry more information through its hidden layers without increasing the number of trainable parameters.
  \item A symmetric architecture is inherently robust to any transformation included in its symmetry group.
\end{enumerate}
In texture-classification experiments with limited training data, \cite{marcos2016learning} found that their rotation-invariant architecture outperformed CNNs trained using rotational data augmentation; this is related to advantage~(1) above.
For advantage~(2), observe that a rotationally augmented CNN tasked with detecting edges would need to learn edge-detecting filters at many orientations, while a rotation-invariant network would only need to learn a single edge-detecting filter~\cite{laptev2016ti,della2019deep}.
For scattering architectures, the robustness in advantage~(3) can extend well beyond the enforced symmetry group. 
In~\cite{sifre2013rotation}, invariance to translations, rotations, and isotropic scaling is combined with deformation stability and learned linear projections to obtain approximate invariance to small shears and other deformations.

Restricting the architecture to enforce symmetry raises an important question about advantage (1): 
Does the smaller search space still contain good approximations to the functions needed for the task? 
Universal approximation results~\cite{ravanbakhsh2020universal} and quantitative approximation bounds~\cite{siegel2026quantitative} provide positive answers for several classes of invariant and equivariant networks. 
But these guarantees depend on the architecture: \cite{pacini2026universality}~identifies certain limitations of shallow equivariant networks, even when they exhibit the distinguishing power of more expressive models. 
See~\cite{blum2025if} for a broader discussion of the relationship between distinguishing power and expressiveness.

\subsection{Overview and Roadmap}\label{sec: overview and road map}

\begin{figure}[t]
    \centering
    \includegraphics[width=\textwidth]{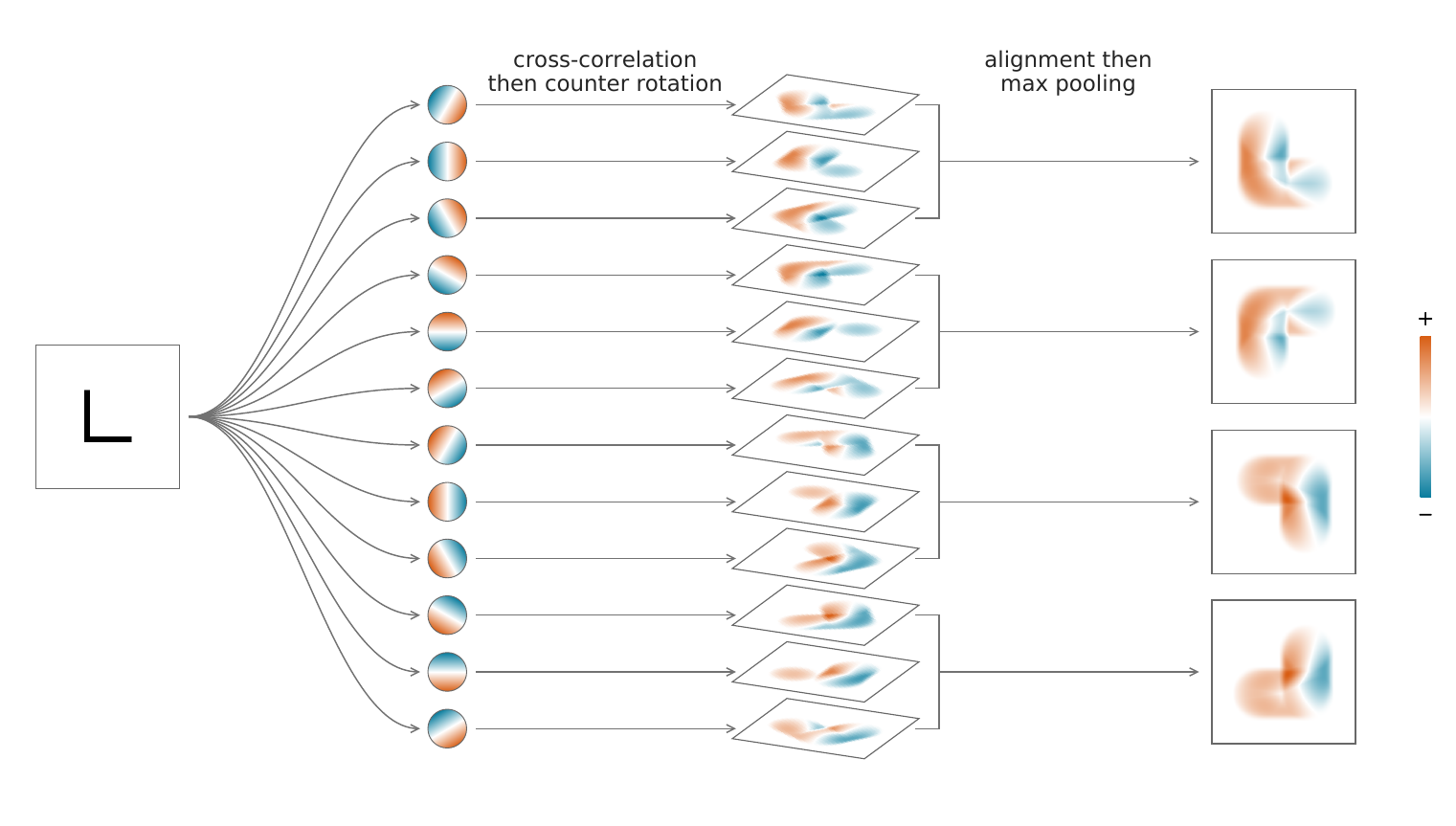}
    \caption{Illustration of a transversal pooling neural network. We process an input image $f$ by first taking its cross-correlation with 12~different rotated versions of a single template before counter-rotating to produce 12~slices of a data cube. This implements the map $\Psi$ with group $G=C_{12}\ltimes\mathbb{R}^2$ (though due to the counter rotation, it is parametrized as $\Psi(f)(A,Ab)$ rather than $\Psi(f)(A,b)$). Next, we batch the slices into neighboring triples, re-rotate the slices in each batch to align with the middle slice, and then compute their pointwise maximum. This max pooling operation corresponds to the map $\Xi$ with subgroup $\Lambda=C_4\ltimes\mathbb{R}^2$. Thanks to the counter-rotated parameterization of the intermediate data cube, rotating the input image by 30~degrees (resp.\ 90~degrees) has the effect of cyclically permuting the 12~intermediate slices (resp.\ four output slices).}
    \label{fig:\ournet{} example}
\end{figure}

\paragraph{Considerations and \ournet{}s.}
Let us return to our motivating list from the beginning of the paper:
\begin{enumerate}
    \item[(a)] in practice, full group invariance can discard relevant information;
    \item[(b)] the standard max pooling nonlinearity in modern neural networks requires careful adaptation to accommodate general affine transformations; and
    \item[(c)] 
    we need quantitative stability guarantees under affine perturbations.
\end{enumerate}
Considerations~(a) and~(b) motivate us to define \emph{transversal pooling neural networks (\ournet{}s)}; see Figure~\ref{fig:\ournet{} example} for an illustration.
Throughout, we let $G$ denote a closed subgroup of the affine general linear group $\AGL(\mathbb{R}^d)$. 
The first layer of a \ournet{} is the wavelet transform $\Psi$ with respect to $G$, defined in~\eqref{eqn:Psi}.
Next, we implement a pooling layer $\Xi$ which pools with respect to a set $H$ of right coset representatives of a closed subgroup $\Lambda \leq G$, or more generally a transversal of $G$:
\[
     \Xi(F)(\lambda) = \sup_{h\in H} F(\lambda h),
     \qquad \lambda\in\Lambda.
\]
 Accordingly, we call $\Xi$ a \emph{transversal pooling} layer.
Given a partition (such as the set of right cosets of a subgroup), a \emph{transversal} is a set containing one element from each member of the partition (such as a set of right coset representatives).
For example, a transversal of the cosets of the integer lattice $\bZ^2 \leq \bR^2$ is the unit square $[0,1)^2$; under this view, the translational max pooling depicted in Figure~\ref{fig:max pooling comparison} is transversal pooling.
For consideration~(c), we prove explicit group-Lipschitz bounds for the composition of the wavelet transform $\Psi$ and transversal pooling $\Xi$ (Theorem~\ref{thm: continuity of xi} and Proposition~\ref{prop: continuity of xiw}) using seemingly new bounds (Theorem~\ref{prop:contest}) on the continuity of the quasiregular representation of $\AGL(\bR^d)$.
After $\Psi$ and $\Xi$, we may compose further group convolutions and activation functions, resulting in a \ournet{} which is equivariant to $\Lambda$ (Corollary~\ref{cor: composition}). 

\paragraph{Contributions.}
The primary contributions of this paper are:
\begin{enumerate}
    \item 
    We introduce a family of neural network architectures, i.e., transversal pooling neural networks (\ournet{}s), that enforce a notion of approximate invariance that emerges in applications.
    \item 
    We prove convergence rates in $\bL^2(\bR^d)$ for the action of the quasiregular (wavelet) representation of subgroups of $\AGL(\bR^d)$.
    \item 
    Focusing on affine subgroups, we prove group-Lipschitz bounds for a variety of neural networks, including CNNs, group-equivariant networks, and \ournet{}s.
    \item 
    We numerically test \ournet{}s on tropical cyclone and synthetic data, demonstrating the effectiveness of this new architecture in data-limited environments.
\end{enumerate}

\paragraph{Roadmap.}
The remainder of the paper is organized as follows.  
In Section~\ref{sec:groupequiv}, we introduce the basic definitions and mathematical background to define \ournet{}s and transversal pooling, proving that they are equivariant to $\Lambda$ in Corollary~\ref{cor: composition}.
The main goal of Section~\ref{sec:transtheory} is to prove that in many cases of interest, \ournet{}s are group-Lipschitz with respect to $G$.
Specifically, in Theorem~\ref{thm: continuity of xi}, we give explicit bounds on 
\[
|\Xi(\Psi(g\cdot f))(\lambda)-\Xi(\Psi(f))(\lambda)|.
\]
The proof of Theorem~\ref{thm: continuity of xi} relies on Theorem~\ref{prop:contest}, which gives a bound on $\norm{T_b D_A \psi - \psi}_{\bL^2(\bR^d)}$ that, to the best of our knowledge, has not appeared elsewhere in the wavelet literature.
Then, in Section~\ref{sec: Group lipschitz pooling}, we apply Theorem~\ref{thm: continuity of xi} to a number of different examples of $G$ and $\Lambda$.
Notable examples include
\begin{itemize}
\item
Example~\ref{ex: cont trans}, which shows that the model of standard CNN max pooling is group-Lipschitz with bounds that depend on the size of the pooling domain, and 
\item
Example~\ref{ex:rot4to2}, which gives a Lipschitz bound for the full rotation-invariant pooling that appears in~\cite{cohen2016group} and related papers.
\end{itemize}
In Section~\ref{sec:affine4realz}, we relax the  mathematical assumptions of Section~\ref{sec:transtheory} to cover the case where $\Lambda$ is the Cartesian product of a dilation subgroup and a translation subgroup (Proposition~\ref{prop: continuity of xiw}).
This approach more closely aligns with the numerical implementation of both discrete wavelets and invariant neural networks; it is also the reason we refer to $\Xi$ as ``transversal pooling'' rather than ``coset-representative pooling.'' 

In Section~\ref{sec: experiments}, we validate our methods on real and synthetic datasets. 
We consider a synthetic dataset of ellipses, which we consider to be a toy model of tropical cyclones (TCs), as they tend to exhibit an approximately elliptical shape.
We use this synthetic dataset to numerically demonstrate our theory (Section~\ref{sec: continuity of pooling}) and to illustrate the effectiveness of \ournet{}s in data-limited environments. 
Finally, Section~\ref{sec: experiments real world} uses \ournet{}s to predict whether a TC will undergo rapid intensification, a machine-learning task described in greater detail in \cite{lagerquist2025identifying}.

\section{Transversal Pooling Neural Networks}\label{sec:groupequiv}
In this section, we present the architecture of transversal pooling neural networks.

\subsection{Setup and Notation}
\label{sec: setup and notation}

Let $\GL(\mathbb{R}^d)$ denote the general linear group over $\mathbb{R}^d$, and consider the affine general linear group
defined by 
$$
\AGL(\mathbb{R}^d)=\GL(\mathbb{R}^d)\ltimes\mathbb{R}^d=\{(x \mapsto Ax+b): A\in \GL(\mathbb{R}^d), b\in \mathbb{R}^d\}.
$$
 Depending on the calculation, we may use $(A,b)$ to denote an element of $\AGL(\mathbb{R}^d)$, emphasizing that the underlying set is $\GL(\mathbb{R}^d)\times\mathbb{R}^d$.
 Alternatively, when knowledge of the components is not necessary, we may use $g$ for an element of $\AGL(\mathbb{R}^d)$.
  We note that $\AGL(\mathbb{R}^d)$ is a group under composition with $(A,b) \cdot (\tilde{A},\tilde{b}) = (A\tilde{A},b + A\tilde{b})$ and $(A,b)^{-1} = (A^{-1},-A^{-1}b)$, and $\AGL(\bR^d)$ acts on $\bR^d$ via $x \mapsto Ax+b$ for $(A,b) \in \AGL(\bR^d)$ and $x \in \bR^d$.
  The action of $\AGL(\bR^d)$ on $\bR^d$ induces an action on $\bL^2(\bR^d)$ known as the \emph{quasiregular} or \emph{wavelet representation}, which was defined above  in~\eqref{eqn:quasireg}, 
  namely
\[
(g\cdot \psi)(x)= \frac{1}{|\det(A)|^{1/2}}\psi(A^{-1}(x-b)).
\]
A simple calculation shows that this action is unitary on $\bL^2(\bR^d)$; in fact, this defines a strongly continuous unitary representation (see, e.g.,~\cite{fuehr2005abstract,ali2000coherent}). 
When we wish to emphasize connections to wavelet theory, we let $D_A$ denote the dilation operator, $D_Af(x)=(A,0)\cdot f(x) = \abs{\det A}^{-1/2}f(A^{-1}x)$ and $T_b$ denote the translation operator, $T_bf(x)=(I,b)\cdot f(x) = f(x-b)$.
Then, $T_b D_A f= (A,b) \cdot f$.
Below in Theorem~\ref{prop:contest}, we will prove a convergence rate of $T_bD_A \psi \rar \psi$ in $\bL^2$ norm as $(A,b) \rar (I,0)$.

Let $G \leq \AGL(\bR^d)$ be closed.
We let $\mu$ denote a left Haar measure on $G$, that is, a nonzero regular Borel measure that is finite on compact sets and satisfies $\mu(g\cdot S)=\mu(S)$ for every $g\in G$ and every Borel $S\subseteq G$, where $g\cdot S=\{g\cdot s:s\in S\}$. Such a measure is unique up to multiplication by a positive constant (see, e.g., \cite{Folland2016,diestel2014joys} for further background).
For $1\leq p<\infty$, we let $\bL^p(G)$ denote the vector space of $\mu$-measurable functions\footnote{As usual, we identify functions that only differ on sets of measure zero.} $F:G\rar\mathbb{R}$ with norm 
$$
\|F\|_{\bL^p(G)}:=\left(\int_G |F(g)|^pd\mu\right)^{1/p}<\infty,
$$
with the usual adjustment for $p = \infty$.
Further, we define $\bC_b(G) = \bL^\infty(G) \cap \bC(G)$, i.e., the bounded continuous functions on $G$.

Define a group action of $G$ on $\bL^p(G)$ using the \emph{regular representation}, i.e., 
\begin{equation}\label{eqn: outer action}
(g\cdot F)(\widetilde{g}) = F(g^{-1}\widetilde{g})
\end{equation} 
for $g,\widetilde{g}\in G$. Further, we note that $\|g\cdot F\|_{\bL^p(G)}=\|F\|_{\bL^p(G)}$ for all $g\in G,$ $F\in \bL^p(G)$, $1\leq p\leq \infty$, which follows from the definition of the Haar measure.
Throughout, we will assume that functions are real-valued, although much of our framework can be extended to complex-valued functions as discussed in Remark~\ref{rem: extend to complex}.

For $R>0$, we let $B(0,R)= \{x\in\mathbb{R}^d: \|x\|_2\leq R\}$ denote the Euclidean ball with radius $R$ centered at the origin and let $\omega_d$ denote the volume of the unit ball $B(0,1)$. For a function $f$ defined on $\mathbb{R}^d$, we let  $\supp(f)$ denote its essential support, i.e.,  the smallest closed set $\Omega$ such that $f=0$ almost everywhere outside of $\Omega$.

A function $f: \bR^d \rar \bR$ is \emph{Lipschitz continuous} if there exists a $C >0$ such that
\[
\abs{f(x)-f(y)} \leq C \norm{x-y}_2
\]
for all $x, y \in \bR^d$.
When $f$ is Lipschitz continuous, we let $\|f\|_{L}$ denote its optimal Lipschitz constant, i.e.,
\[
\norm{f}_L = \sup_{x \neq y} \frac{\abs{f(x)-f(y)}}{\norm{x-y}_2}.
\]
Note that $\|\cdot\|_L$ is a seminorm (which explains the notation). 
For a matrix $A \in \GL(\bR^d)$, we will let $\det(A)$ denote its determinant and let $\|A\|_{2\rar 2}$ denote its operator norm on $\ell^2$, i.e., $\|A\|_{2\rar 2}=\sup_{x\neq 0}\frac{\|Ax\|_2}{\|x\|_2}$. Further, we let $e^A$ denote the matrix exponential. 
Finally, we let $\bW^{1,2}(\mathbb{R}^d)$ denote the Sobolev space consisting of weakly differentiable functions such that 
$$
\int_{\mathbb{R}^d} \big(|f(x)|^2 + \|\nabla f(x)\|_2^2\big) dx <\infty.
$$

\subsection{Group Equivariant Filter Banks and Transversal Pooling}\label{sec:GEfilterbanks}
In this section, we will let $G$ be a closed subgroup of $\AGL(\mathbb{R}^d)$, let  $\Lambda$ be a closed subgroup of $G$, and assume that $G/\Lambda$ is compact.
Inspired by wavelets and scattering transforms, we  consider a cascade of three key operators, beginning with the continuous wavelet transform over $G$:
\[
\bL^2(\mathbb{R}^d) ~\xrightarrow{\quad\Psi\quad}~ \mathbf{C}_b(G) ~\xrightarrow{\quad\Xi\quad}~ \bC_b(\Lambda) ~\xrightarrow{\quad\Upsilon\quad}~ \bC_b(\Lambda).
\]
Here, $\Psi$ is the continuous wavelet transform (voice transform) defined in~\eqref{eqn:Psi}, $\Xi$ is a max pooling--type operator, and $\Upsilon$ is a group convolution operator.
In Section~\ref{sec: equivariant neural network architecturs}, we will introduce neural networks based on these operators, which may also use nonlinear activation functions $\sigma$ between the operators $\Psi$, $\Xi$, and $\Upsilon$. 
We first recall the definition of the voice transform (continuous wavelet transform) $\Psi$ from~\eqref{eqn:Psi}: for $f, \psi\in\bL^2(\mathbb{R}^d)$,
\[
\Psi(f)(g)= \langle f, g\cdot \psi\rangle_{\bL^2(\mathbb{R}^d)}=\langle f, \psi_g\rangle_{\bL^2(\mathbb{R}^d)},
\]
where we  denote $g \cdot \psi=\psi_g$.

In this work, we take the codomain of $\Psi$ to be $\mathbf{C}_b(G)$, the set of bounded, continuous functions defined on $G$, a choice motivated by the fact that the output of $\Psi$ is passed to the max pooling operator~$\Xi$. 
The fact that $\Psi(f)$ is bounded follows from an application of the Cauchy--Schwarz inequality:
\[
|\Psi(f)(g)|
=|\langle f, \psi_g\rangle_{\bL^2(\mathbb{R}^d)}|
\leq \|f\|_{\bL^2(\mathbb{R}^d)}\|\psi_g\|_{\bL^2(\mathbb{R}^d)}
=\|f\|_{\bL^2(\mathbb{R}^d)}\|\psi\|_{\bL^2(\mathbb{R}^d)},
\]
where the final equality follows from a change of variables.
Meanwhile, the continuity of $\Psi(f)$ follows from the fact that the quasiregular representation is strongly continuous.
Later, in Theorem~\ref{prop:contest}, we prove explicit rates of the strong continuity for the quasiregular representation. 

In many cases of interest, $\Psi$ maps $\mathbf{L}^2(\mathbb{R}^d)$ to $\mathbf{L}^2(G) \cap \bC_b(G)$ as either an isometric embedding (strong admissibility) or a bounded injective map (weak admissibility).
In fact, admissibility is often considered a requirement for $\psi$ to be a \emph{wavelet};
see~\cite{calderon1964intermediate,duflo1976regular,murenzi1989wavelet,LWWW02,fuehr2002continuous} for more information.

The following result shows that $\Psi$ is $G$-equivariant. 
As with Proposition~\ref{prop: upsilon equivariance} below (which establishes the equivariance of $\Upsilon$), it is well-known in abstract harmonic analysis (see e.g.,~\cite{duflo1976regular,grossmann1985transforms,feichtinger1989banach,Folland2016}), 
but we provide proofs in Appendix~\ref{app: proofs of standard props} in order to keep our work self-contained. 

\begin{prop}[Equivariance of $\Psi$]\label{prop: Psi Equivariant}
$\Psi$ is $G$-equivariant from $\bL^2(\mathbb{R}^d)$ to $\mathbf{C}_b(G)$,
i.e., for  all $f\in\mathbf{L}^2(\mathbb{R}^d)$  and $g\in G$, we have 
\begin{equation*}
g\cdot\Psi(f)=\Psi(g\cdot f),
\end{equation*}
where the actions of $G$ on $\bL^2(\mathbb{R}^d)$ and $\mathbf{C}_b(G)$  are defined as in~\eqref{eqn:quasireg} and~\eqref{eqn: outer action}.
\end{prop}

Now, we consider the closed subgroup $\Lambda$ and  recall that $G/\Lambda$ is assumed to be compact. We let $H$ be a fixed precompact set of right coset representatives (which exists since $G$ is locally compact and $G/\Lambda$ is compact). For $1\leq p\leq \infty$, we let $\Lambda$ act on $\bL^p(\Lambda)$ in the manner inherited from~\eqref{eqn: outer action}, i.e.,  for $\lambda, \widetilde{\lambda}\in \Lambda$, we let
 \begin{equation}\label{eqn: outer action Lambda}
(\lambda \cdot F)(\widetilde{\lambda})= F({\lambda}^{-1}\widetilde{\lambda}).
\end{equation} 
We define
$\Xi:\mathbf{C}_b(G)\rar\bC_b(\Lambda)$ by 
 \begin{equation}
     \label{eqn: pooling}
     \Xi(F)(\lambda) = \sup_{h\in H} F(\lambda h),
 \end{equation}
 and will refer to $\Xi$ as a  \emph{transversal pooling} operator. 
 To better understand the definition~\eqref{eqn: pooling}, consider the case where $G=\bR^d$ is the translation group.
 We fix $c>0$, let $\Lambda$ be a lattice of the form
 \[
 c\mathbb{Z}^d=\{(cm_1,\ldots,c m_d):m_i\in\mathbb{Z}\}, 
 \]
 and choose $H=\left[-\frac{c}{2},\frac{c}{2}\right)^d$.
In this case, we have 
$$
\Xi(F)(\lambda)=\sup_{h\in[-c/2,c/2)^d}F(\lambda+h),
$$
i.e., the transversal pooling $\Xi$ can be viewed as a continuum analog of max pooling operators commonly used in convolutional neural networks. 
This connection is illustrated in Figure~\ref{fig:max pooling comparison}.
The left-hand side of the figure represents the $6 \times 4$ grid of numbers output from a discrete convolutional layer along with the $3 \times 2$ grid of numbers obtained from pooling over disjoint $2 \times 2$ neighborhoods.
The right-hand side shows a rectangular subdomain of the output $\mathbb{R}^2\to\mathbb{R}$ of $\Psi$ and of the corresponding output $\mathbb{Z}^2\to\mathbb{R}$ of $\Xi$.

\begin{thm}[Equivariance of $\Xi$]\label{prop: xi equivariance}
    $\Xi$ is $\Lambda$-equivariant from $\mathbf{C}_b(G)$ to $\bC_b(\Lambda)$, i.e., for all $F\in\mathbf{C}_b(G)$ and $\lambda\in\Lambda$, we have  
$$ \Xi(\lambda\cdot F)={\lambda}\cdot \Xi(F),
$$
 where the actions of $\Lambda$ on $\mathbf{C}_b(G)$ and $\bC_b(\Lambda)$  are defined as in~\eqref{eqn: outer action} and~\eqref{eqn: outer action Lambda}.
\end{thm}

\begin{proof}
Let $F\in \bC_b(G)$ and $\lambda,\widetilde{\lambda}\in\Lambda$. Then
\[
\Xi({\lambda}\cdot F)(\widetilde\lambda)
=\sup_{h\in H}({\lambda}\cdot F)(\widetilde\lambda h)=\sup_{h\in H}F({\lambda}^{-1}\widetilde\lambda h)=\Xi(F)({\lambda}^{-1}\widetilde\lambda)
=\big({\lambda}\cdot \Xi(F)\big)(\widetilde\lambda).
\qedhere
\]
\end{proof}

After applying $\Psi$ and $\Xi$, we may wish to apply another convolution defined on ${\Lambda}$. 
Towards this end, we let $\upsilon\in\mathbf{L}^1(\Lambda)$
and define $\Upsilon:\bC_b(\Lambda)\rar \bC_b(\Lambda)$ by
\begin{equation}\label{eqn: upsilon}
\Upsilon(F)(\lambda)=\int_\Lambda F(\kappa)\upsilon(\lambda^{-1}\kappa)d\mu_{\Lambda}(\kappa),
\end{equation}
where $\mu_\Lambda$ denotes the left Haar measure on $\Lambda$. 
The following result, which establishes the equivariance of~$\Upsilon$, is a standard fact about convolution over Haar measure (see e.g., \cite{Folland2016}). We provide a short proof in Appendix~\ref{app: proof of upsilon equivarariance} to help keep our work self-contained.

\begin{prop}[Equivariance of $\Upsilon$]\label{prop: upsilon equivariance}
    Let $\upsilon \in \bL^1(\Lambda)$. Then $\Upsilon$ is $\Lambda$-equivariant from $\bC_b(\Lambda)$ to $\bC_b(\Lambda)$, i.e., for all $F\in\bC_b(\Lambda)$ and $\lambda\in\Lambda$, we have 
$$ \Upsilon(\lambda\cdot F)=\lambda\cdot \Upsilon(F),
$$
where the action of $\Lambda$ on  $\bC_b(\Lambda)$  is defined as in~\eqref{eqn: outer action Lambda}.
\end{prop}

\subsection{Activation Functions and Transversal Pooling Neural Networks}\label{sec: equivariant neural network architecturs}

In this section, we define a special class of neural network architectures by composing the operators $\Psi$, $\Xi$, and $\Upsilon$, along with pointwise activations.
To be clear, given an \emph{activation function} $\sigma:\mathbb{R}\rar\mathbb{R}$, we write
\[
\sigma(f)(x)=\sigma(f(x))
\]
for $f$ in $\mathbf{L}^2(\mathbb{R}^d)$ or in $\bC_b(\Gamma)$ for some $\Gamma \leq G$.
(That is, in a minor abuse of notation, we use $\sigma$ to denote both the scalar function and the associated operator.)
We take our activation functions to be
\begin{itemize}
\item
\emph{positively homogeneous}, meaning $\sigma(cx)=c\sigma(x)$ for all $x\in\bR$ and $c\geq 0$, and
\item
\emph{weakly contractive}, meaning $|\sigma(x)-\sigma(y)|\leq |x-y|$ for all $x,y\in\bR$.
\end{itemize}
We observe that the associated operator on functions is similarly positively homogeneous and weakly contractive in the relevant norm.
Many of the most popular activation functions, including ReLU and Leaky-ReLU~\cite{maas2013rectifier}, satisfy these requirements, as does the absolute value commonly used in scattering transforms and associated neural networks.

When $\sigma$ is positively homogeneous, the corresponding operator on $\mathbf{L}^2(\mathbb{R}^d)$ is equivariant to the quasiregular representation of $\operatorname{AGL}(\bR^d)$:
\[
(g\cdot \sigma(f))(x)
=\frac{1}{|\det(Jg)|^{1/2}}\sigma(f(g^{-1}(x)))
=\sigma\left(\frac{1}{|\det(Jg)|^{1/2}}f(g^{-1}(x))\right)
=\sigma(g\cdot f)(x).
\]
Here, $Jg$ denotes the Jacobian matrix of $g$, and the middle step applies the positive homogeneity of $\sigma$. 
By a similar argument, the corresponding operator on $\bC_b(\Gamma)$ is equivariant to the regular representation of~$\Gamma$, even when $\sigma$ is not positively homogeneous (cf.\ \cite[Section 6.2]{cohen2016group}).

We now define a \emph{transversal pooling neural network (\ournet{})} to be a composition
\begin{equation}\label{eqn: mathcal T definition}
\mathcal{T}
=\tau_m\circ\tau_{m-1}\circ\cdots\circ\tau_1\circ \Psi,
\end{equation}
where each $\tau_i$ is a function of the form $\Xi$, $\Upsilon$, or $\sigma$, and we allow the subgroups $\Lambda$, the transversals $H$, the filters $\upsilon$, and the activation functions $\sigma$ to vary from one layer to another.
In order for $\mathcal{T}$ to be well-defined, we require the codomain of each layer to equal the domain of the subsequent layer.
We also refer to any implementation of this structure in the digital domain, e.g., our experiments in Section~\ref{sec: experiments}, as a \ournet{}.
Since the composition of equivariant maps is equivariant, we may combine the equivariance of~$\sigma$ with Proposition~\ref{prop: Psi Equivariant}, Theorem~\ref{prop: xi equivariance}, and Proposition~\ref{prop: upsilon equivariance} to obtain the following.

\begin{cor}\label{cor: composition}
A transversal pooling neural network $\mathcal{T}: \bL^2(\mathbb{R}^d)\rar\bC_b(\Lambda)$  defined as in~\eqref{eqn: mathcal T definition} is $\Lambda$-equivariant, i.e., for  all $f\in\mathbf{L}^2(\mathbb{R}^d)$ and $\lambda\in\Lambda$, we have 
 $$ \mathcal{T}(\lambda\cdot f)=\lambda\cdot (\mathcal{T}f), $$
where the actions of $\lambda$ on $\bL^2(\mathbb{R}^d)$ and $\bC_b(\Lambda)$  are defined as in~\eqref{eqn:quasireg} and~\eqref{eqn: outer action Lambda}.
\end{cor}

\begin{remark}\label{rem: extend to complex}

Throughout, we have assumed that $f$ is real-valued, both for convenience and because this is typically the case in the setting of neural networks. 
However, our definitions can be adapted for complex-valued functions. 
Extending the definition of $\Psi$ is straightforward: we simply use the complex inner product $\langle f,g\rangle_{\mathbf{L}^2(\mathbb{R}^d)}=\int_{\mathbb{R}^d}f(x)\overline{g(x)}dx$. 
For the activation functions $\sigma$, we may apply $\sigma$ to the real and imaginary parts separately or choose $\sigma$ to be the complex modulus as done in \cite{mallat:scattering2012}. Alternatively, as in \cite{Zhang2021MagNet}, we could use the complex version of  ReLU defined by $\sigma(z)=z$, if $-\pi/2\leq \arg(z)<\pi/2$, and $\sigma(z)=0$, although this would require  some modification of our continuity theory introduced in Section \ref{sec:transtheory} since the complex ReLU is not continuous. It is straightforward to check that each of these choices are positively homogeneous. 
For the transversal pooling operator $\Xi$, there are multiple options.
For example, we could take the supremum of the real and imaginary parts separately and have the output of $\Xi$ be complex-valued, or we could take the supremum of $|F(\lambda h)|$ and have the output of $\Xi$ be a real-valued function. 
\end{remark}

\section{Group-Lipschitz Continuity of Transversal Pooling}\label{sec:transtheory}

\subsection{Group-Lipschitz Continuity and Wavelet Representation Convergence Rates}\label{sec:thyGL}

In this section, we establish that our transversal pooled filter bank  $\Xi\circ \Psi$ is group-Lipschitz with respect to~$G$, so that $\Xi(\Psi(g\cdot f))\approx \Xi(\Psi(f))$ when $g\in G$ is close to the identity; see Theorem~\ref{thm: continuity of xi}.
We start by proving the following prerequisite result, which gives quantitative bounds on the continuity of the quasiregular representation of $\AGL(\bR^d)$, and may be of independent interest. 
\begin{thm}\label{prop:contest}
Let $g=(A,b)\in \AGL(\mathbb{R}^d)$, and let $\psi \in \bW^{1,2}(\bR^d)$. Assume $\det(A) >0$ 
and that
\[
\big{\|}\norm{\cdot}_2 \nabla \psi\big{\|}_{\bL^2(\bR^d)}
:=\left(\int_{\mathbb{R}^d}\|x\|^2_2\|\nabla \psi(x)\|_2^2dx\right)^{1/2}< \infty.
\]
It holds that
\begin{equation}
\label{eq: affine bound}
\norm{g\cdot\psi - \psi}_{\bL^2(\bR^d)} \leq C_A \big{\|}\norm{\cdot}_2 \nabla \psi\big{\|}_{\bL^2(\bR^d)}+\frac{1}{2}\abs{\ln \det A}\norm{\psi}_{\bL^2(\bR^d)}  +\norm{b}_2 \norm{\nabla \psi}_{\bL^2(\bR^d)},
\end{equation}
where
\begin{equation}\label{eqn: CA}
C_A = \max\left\{\ln\norm{A}_{2 \rar 2}, \ln\norm{A^{-1}}_{2 \rar 2}\right\}+ 2\arcsin\left( \frac{1}{2}\norm{A (A^\top A)^{-1/2}-I}_{2 \rar 2}\right).
\end{equation}
If we further assume that $\psi$ is Lipschitz and compactly supported with Lipschitz constant $\norm{\psi}_{L}$ and \\$\supp(\psi) \subseteq B(0,R)$, then
\begin{equation}\label{eqn: prop compact part}
\norm{g\cdot\psi - \psi}_{\bL^2(\bR^d)}\leq \left(R C_A+\norm{b}_2 \right)R^{d/2}\omega_d^{1/2}\norm{\psi}_{L}+\frac{1}{2}\abs{\ln \det A}\norm{\psi}_{\bL^2(\bR^d)},
\end{equation}
where $\omega_d$ is the volume of the unit ball in $\bR^d$.
\end{thm}

To help parse Theorem~\ref{prop:contest}, first note that any $A$ that is close to $I$ has positive determinant. 
Next, $C_A$ quantifies how different $A$ is from $I$; for instance, $C_I = 0$. Therefore, Theorem~\ref{prop:contest} measures the local continuity of the map $g\rar g\cdot\psi$ in terms of how far $g$ is from the identity element in $\AGL(\mathbb{R}^d)$. 
Our proof of Theorem~\ref{prop:contest} makes use of the following lemma.

\begin{lem}\label{lem: semigroup bound}
Suppose $f\in \bW^{1,2}(\bR^d)$ and
$\big\|\norm{\cdot}_2\nabla f\big\|_{\bL^2(\bR^d)}<\infty$.
For every $M\in\mathbb{R}^{d\times d}$, it holds that
\[
\norm{e^M\cdot f-f}_{\bL^2(\bR^d)}
\leq
\norm{\frac{1}{2}\trace(M) f+\ip{\nabla f}{M\cdot}}_{\bL^2(\bR^d)}.
\]
\end{lem}

\begin{proof}
The set of functions that satisfy the hypotheses forms a Banach space $X$ with norm
\[
\|f\|_X
=\|f\|_{\bL^2(\bR^d)}+\|\nabla f\|_{\bL^2(\bR^d)}+\big\|\norm{\cdot}_2\nabla f\big\|_{\bL^2(\bR^d)}.
\]
One may show that $\bC_c^\infty(\bR^d)$ is dense in $X$ by cutting off and mollifying.
Since the left- and right-hand sides of the desired inequality are continuous over $X$, we may assume $f\in\bC_c^\infty(\bR^d)$ without loss of generality.

Now define $F:[0,1]\to\bL^2(\bR^d)$ by $F(t):=e^{tM}\cdot f$, i.e.,
\[
F(t)(x)
=e^{-t\trace(M)/2}f(e^{-tM}x).
\]
By the fundamental theorem of calculus, it suffices to bound the norm of
\[
e^M\cdot f-f
=F(1)-F(0)
=\int_0^1 F'(t)\,dt,
\]
where $F'(t)(x):=\frac{\partial}{\partial t}(F(t)(x))$.
The product rule and chain rule together give
\[
F'(t)(x)
=e^{-t\trace(M)/2}
\left(
-\frac{1}{2}\trace(M)f(e^{-tM}x)
+\ip{\nabla f(e^{-tM}x)}
{-Me^{-tM}x}
\right)
=(e^{tM}\cdot g)(x),
\]
where
\[
g(x):=-\frac{1}{2}\trace(M)f(x)
+\ip{\nabla f(x)}{-Mx}.
\]
Then Minkowski's integral inequality gives
\[
\norm{e^M\cdot f-f}_{\bL^2(\bR^d)}
\leq\int_0^1\norm{F'(t)}_{\bL^2(\bR^d)}\,dt
=\int_0^1\norm{e^{tM}\cdot g}_{\bL^2(\bR^d)}\,dt
=\norm{g}_{\bL^2(\bR^d)},
\]
where the last step uses the fact that $e^{tM}$ acts isometrically on $\bL^2(\bR^d)$.
\end{proof}

\begin{proof}[Proof of Theorem~\ref{prop:contest}]
Put $P=(A^\top A)^{1/2}$ and $U=AP^{-1}$.
Since $\det A>0$, the matrix $P$ is positive definite and
$U\in\operatorname{SO}(d)$. 
Then we can decompose $g$ as
\[
g=(I,b)(U,0)(P,0),
\]
and so the triangle inequality and unitarity together give
\begin{equation}\label{eq: three components}
\norm{g\cdot \psi-\psi}_{\bL^2(\bR^d)}
\leq\norm{P\cdot\psi-\psi}_{\bL^2(\bR^d)}
+\norm{U\cdot\psi-\psi}_{\bL^2(\bR^d)}
+\norm{(I,b)\cdot\psi-\psi}_{\bL^2(\bR^d)}.
\end{equation}
Since $P$ is positive definite, its spectral decomposition gives
$S:=\log P$, which is real and symmetric.
Applying Lemma~\ref{lem: semigroup bound} and the Cauchy--Schwarz
inequality then gives
\[
\norm{P\cdot\psi-\psi}_{\bL^2(\bR^d)}
\leq\frac{1}{2}\abs{\trace(S)}\norm{\psi}_{\bL^2(\bR^d)}
+\norm{S}_{2\to2}\big\|\norm{\cdot}_2\nabla\psi\big\|_{\bL^2(\bR^d)}.
\]
Next, by the maximal torus theorem, there exists a logarithm
$K\in\mathfrak{so}(d)$ of $U\in\operatorname{SO}(d)$ with $\norm{K}_{2\to2}\leq\pi$.
The skew-symmetric matrix $K$ has trace zero, and so Lemma~\ref{lem: semigroup bound} and Cauchy--Schwarz
again give
\[
\norm{U\cdot\psi-\psi}_{\bL^2(\bR^d)}
\leq\norm{\ip{\nabla\psi}{K\cdot}}_{\bL^2(\bR^d)}
\leq\norm{K}_{2\to2}\big\|\norm{\cdot}_2\nabla\psi\big\|_{\bL^2(\bR^d)}.
\]
For the translation by $b\in\bR^d$, the standard Sobolev estimate
\cite[Proposition~9.3]{brezis2011functional} yields
\[
\norm{(I,b)\cdot\psi-\psi}_{\bL^2(\bR^d)}
\leq\norm{b}_2\norm{\nabla\psi}_{\bL^2(\bR^d)}.
\]
To identify the constants, the eigenvalues of $S$ and the rotation
angles of $U$ give
\[
\trace(S)=\ln\det A,
\qquad
\norm{S}_{2\to2}
=\max\big\{\ln\norm{A}_{2\to2},\,
\ln\norm{A^{-1}}_{2\to2}\big\},
\]
\[
\norm{K}_{2\to2}
=2\arcsin\left(\frac{1}{2}\norm{U-I}_{2\to2}\right).
\]
Thus $\norm{S}_{2\to2}+\norm{K}_{2\to2}=C_A$.
Substituting the three estimates into~\eqref{eq: three components} yields~\eqref{eq: affine bound}.
Finally, if $\psi$ is Lipschitz and supported in $B(0,R)$, then
\[
\norm{\nabla\psi}_{\bL^2(\bR^d)}
\leq R^{d/2}\omega_d^{1/2}\norm{\psi}_L,
\qquad
\big\|\norm{\cdot}_2\nabla\psi\big\|_{\bL^2(\bR^d)}
\leq R^{1+d/2}\omega_d^{1/2}\norm{\psi}_L,
\]
and~\eqref{eqn: prop compact part} follows.
\end{proof}

We now present the main result of this section. 
Here, it is convenient to denote $\psi_g:=g\cdot \psi$.

\begin{thm}[Group-Lipschitz Continuity of $\Xi\circ \Psi$]\label{thm: continuity of xi}
Let $g=(A,b)\in G$, and let $\psi \in \bW^{1,2}(\bR^d)$. Assume $\det(A) >0$ and that  $\big{\|}\norm{\cdot}_2 \nabla \psi\big{\|}_{\bL^2(\bR^d)}< \infty$. 
Then for every $f\in\bL^2(\mathbb{R}^d)$ and $\lambda\in\Lambda$, we have
\begin{align*}
\lefteqn{|\Xi(\Psi(g\cdot f))(\lambda)-\Xi(\Psi(f))(\lambda)| }\\
&\leq \sup_{h\in H}\left(C_{A} \big{\|}\norm{\cdot}_2 \nabla \psi_{\lambda h}\big{\|}_{\bL^2(\bR^d)}+\frac{1}{2}\abs{\ln \det A^{}} \norm{\psi_{\lambda h}}_{\bL^2(\bR^d)}  +\norm{A^{-1}b}_2 \norm{\nabla \psi_{\lambda h}}_{\bL^2(\bR^d)}\right)\cdot\|f\|_{\bL^2(\mathbb{R}^d)},
\end{align*}
where $C_A$ is defined as in~\eqref{eqn: CA}.
If we further assume that $\psi$ is Lipschitz and compactly supported with Lipschitz constant $\norm{\psi}_{L}$ and $\supp(\psi) \subseteq B(0,R)$, then for every $f\in\mathbf{L}^2(\mathbb{R}^d)$  and $\lambda\in\Lambda$, we have
\begin{align}
\lefteqn{|\Xi(\Psi(g\cdot f))(\lambda)-\Xi(\Psi(f))(\lambda)|}\nonumber\\
&\leq \left(\sup_{h \in H} \left(\tilde{R}_{\lambda h} C_{A}+\norm{A^{-1}b}_2 \right)\tilde{R}_{\lambda h}^{d/2}\omega_d^{1/2}\norm{\psi_{\lambda h}}_{L}+\frac{1}{2}\abs{\ln \det A} \norm{\psi}_{\bL^2(\bR^d)}\right)\cdot\|f\|_{\bL^2(\mathbb{R}^d)},\label{eqn: theorem result compact case}
\end{align} 
where $\omega_d$ is the volume of the unit ball in $\mathbb{R}^d$,  
$A_{\lambda h}$ and $b_{\lambda h}$ are the matrix and vector such that $(\lambda h)(x)=A_{\lambda h}x+b_{\lambda h}$, and $\tilde{R}_{\lambda h} = R\norm{A_{\lambda h}}_{2 \rar 2} + \norm{b_{\lambda h}}_2$.
\end{thm}

The proof of Theorem~\ref{thm: continuity of xi} relies on the following lemma.

\begin{lem}\label{lem: continuity simple}
Let $g=(A,b)\in G$, and assume that $\psi\in\bL^2(\mathbb{R}^d)$.  Then, for $f\in\bL^2(\mathbb{R}^d)$ and $\lambda\in\Lambda$, we have
\[
|\Xi(\Psi(g\cdot f))(\lambda)-\Xi(\Psi(f))(\lambda)|
\leq \|f\|_{\bL^2(\mathbb{R}^d)} \sup_{h\in H}\|\psi_{g^{-1}\lambda h}-\psi_{\lambda h}\|_{\bL^2(\mathbb{R}^d)}.
\]
\end{lem}
\begin{proof}
Applying Proposition~\ref{prop: Psi Equivariant} gives
\[
\Xi(\Psi(g\cdot f))(\lambda) =\Xi(g\cdot \Psi(f))(\lambda)
=\sup_{h\in H} \Psi(f)(g^{-1}\lambda h)
=\sup_{h\in H} \langle f,\psi_{g^{-1}\lambda h}\rangle_{\bL^2(\mathbb{R}^d)},
\]
and similarly, $\Xi(\Psi(f))(\lambda) =\sup_{h\in H} \langle f,\psi_{\lambda h}\rangle$.  
Therefore,
\begin{align*}
|\Xi(\Psi(g\cdot f))(\lambda)-\Xi(\Psi(f))(\lambda)|
&=\left|\sup_{h\in H} \langle f,\psi_{g^{-1}\lambda h}\rangle_{\bL^2(\mathbb{R}^d)}-\sup_{h\in H} \langle f,\psi_{\lambda h}\rangle_{\bL^2(\mathbb{R}^d)}\right|\nonumber \\
&\leq \sup_{h\in H}\left| \langle f,\psi_{g^{-1}\lambda h}\rangle_{\bL^2(\mathbb{R}^d)}- \langle f,\psi_{\lambda h}\rangle_{\bL^2(\mathbb{R}^d)}\right|\nonumber \\
&\leq \|f\|_{\bL^2(\mathbb{R}^d)} \sup_{h\in H}\|\psi_{g^{-1}\lambda h}-\psi_{\lambda h}\|_{\bL^2(\mathbb{R}^d)},
\end{align*}
where the last step applies the Cauchy--Schwarz inequality.
\end{proof}
\begin{remark}\label{rem: Proof of Lemma Continuity Simple}
The proof of Lemma~\ref{lem: continuity simple} above uses the $G$-equivariance of $\Psi$, but importantly, it does not use the $\Lambda$-equivariance of $\Xi$.
This will allow us to establish Proposition~\ref{prop: continuity of xiw} below, which shows that a bound analogous to Theorem~\ref{thm: continuity of xi} holds even in certain cases where $\Lambda$ is not a subgroup of $G$.
\end{remark}

\begin{proof}[Proof of Theorem~\ref{thm: continuity of xi}]
For $\lambda\in\Lambda$, $h\in H$, 
we apply Theorem~\ref{prop:contest} with $\psi_{\lambda h}$ in place of $\psi$ and $g^{-1}$ in place of $g$, and use the facts that $C_{A^{-1}}=C_A$ and $\ln \det A^{-1}=-\ln \det A$ to obtain
\begin{align*}
\|\psi_{g^{-1}\lambda h}-\psi_{\lambda h}\|_{\bL^2(\mathbb{R}^d)}
&\leq C_{A^{-1}} \big{\|}\norm{\cdot}_2 \nabla \psi_{\lambda h}\big{\|}_{\bL^2(\bR^d)}+\frac{1}{2}\abs{\ln \det A^{-1}} \norm{\psi_{\lambda h}}_{\bL^2(\bR^d)}  +\norm{A^{-1}b}_2 \norm{\nabla \psi_{\lambda h}}_{\bL^2(\bR^d)}
\\ &\leq C_{A} \big{\|}\norm{\cdot}_2 \nabla \psi_{\lambda h}\big{\|}_{\bL^2(\bR^d)}+\frac{1}{2}\abs{\ln \det A^{}} \norm{\psi_{\lambda h}}_{\bL^2(\bR^d)}  +\norm{A^{-1}b}_2 \norm{\nabla \psi_{\lambda h}}_{\bL^2(\bR^d)}.
\end{align*}
Taking the supremum over $h$ yields 
\begin{align}\label{eqn: thoerem not nice}
\lefteqn{\sup_{h\in H}\|\psi_{g^{-1}\lambda h}-\psi_{\lambda h}\|_{\bL^2(\mathbb{R}^d)}}\nonumber\\
&\leq \sup_{h\in H}\left(C_{A} \big{\|}\norm{\cdot}_2 \nabla \psi_{\lambda h}\big{\|}_{\bL^2(\bR^d)}+\frac{1}{2}\abs{\ln \det A^{}} \norm{\psi_{\lambda h}}_{\bL^2(\bR^d)}  +\norm{A^{-1}b}_2 \norm{\nabla \psi_{\lambda h}}_{\bL^2(\bR^d)}\right),\nonumber
\end{align}
and so  the first claim follows from Lemma~\ref{lem: continuity simple}.

For the second claim, note that for any $\lambda\in\Lambda$ and $h\in H$,
\[
\supp (\psi_{\lambda h}) \subset B(b_{\lambda h}, R\norm{A_{\lambda h}}_{2 \rar 2}) \subset B(0, R\norm{A_{\lambda h}}_{2 \rar 2}+\norm{b_{\lambda h}}_2) = B(0,\tilde{R}_{\lambda h}).
\]
Therefore, applying~\eqref{eqn: prop compact part} with $\psi_{\lambda h}$ in place of $\psi$, $\widetilde R_{\lambda h}$ in place of $R$, and $g^{-1}$ in place of $g$, we obtain
\begin{align}
\sup_{h\in H}\|\psi_{g^{-1}\lambda h}-\psi_{\lambda h}\|_{\bL^2(\mathbb{R}^d)}
&\leq \sup_{h \in H} \left(\frac{1}{2}\abs{\ln \det A^{-1}} \norm{\psi_{\lambda h}}_{\bL^2(\bR^d)} + \left(\tilde{R}_{\lambda h} C_{A^{-1}}+\norm{-A^{-1}b}_2 \right)\tilde{R}_{\lambda h}^{d/2}\omega_d^{1/2}\norm{\psi_{\lambda h}}_{L}\right)\nonumber\\
&= \sup_{h \in H} \left(\frac{1}{2}\abs{\ln \det A} \norm{\psi}_{\bL^2(\bR^d)} + \left(\tilde{R}_{\lambda h} C_{A}+\norm{A^{-1}b}_2 \right)\tilde{R}_{\lambda h}^{d/2}\omega_d^{1/2}\norm{\psi_{\lambda h}}_{L}\right)\label{eqn: switch the Cs}\nonumber\\
&=  \frac{1}{2}\abs{\ln \det A} \norm{\psi}_{\bL^2(\bR^d)} +\sup_{h \in H} \left(\tilde{R}_{\lambda h} C_{A}+\norm{A^{-1}b}_2 \right)\tilde{R}_{\lambda h}^{d/2}\omega_d^{1/2}\norm{\psi_{\lambda h}}_{L}\nonumber,
\end{align}
and we conclude by again applying Lemma~\ref{lem: continuity simple}.
\end{proof}

\subsection{Examples of Group-Lipschitz Continuity}\label{sec: Group lipschitz pooling}

\begin{table}[tb]
\centering
\small
\renewcommand{\arraystretch}{1.3}
\setlength{\tabcolsep}{10pt}
\begin{tabular}{@{}llllcl@{}}
\hline
Name
& $G$
& $\Lambda_D$
& $\Lambda_T$
& $\Lambda\leq G$?
& Location \\
\hline

Rotations
& $\SO(d)$
& $\Lambda_D\leq \SO(d)$
& $\{0\}$
& \ding{51}
& Example~\ref{ex: cont rotations} \\

Translations
& $\bR^d$
& $\{I\}$
& $c\bZ^d$
& \ding{51}
& Example~\ref{ex: cont trans} \\

Isotropic dilations
& $\{sI:s>0\}$
& $a^{\bZ}I$
& $\{0\}$
& \ding{51}
& Example~\ref{ex: cont dilations} \\

Rototranslations
& $\SO(2)\ltimes\bR^2$
& $C_4$
& $c\bZ^2$
& \ding{51}
& Example~\ref{ex:rot4} \\

Discrete rototranslations
& $C_4\ltimes\bZ^2$
& $\{I\}$
& $c\bZ^2$
& \ding{51}
& Example~\ref{ex:rot4to2} \\

Direct similitudes
& $(\bR_{>0}\times\SO(d))\ltimes\bR^d$
& $a^{\bZ}\Lambda_r$, $\Lambda_r \leq \SO(d)$
& $c\bZ^d$
& \ding{55}
& Example~\ref{ex:dirsim} \\
\hline
\end{tabular}
\caption{Ambient transformation groups $G$ and retained index sets
$\Lambda=\{(A,b):A\in\Lambda_D,\ b\in\Lambda_T\}$
in the examples of
Sections~\ref{sec: Group lipschitz pooling}
and~\ref{sec:affine4realz}.}
\label{tab:transversal-pooling-examples}
\end{table}

In this section, we apply Theorem~\ref{thm: continuity of xi} to several examples of interest; see Table~\ref{tab:transversal-pooling-examples} for a summary.
For each group $G$, we seek a right-invariant metric $\operatorname{dist}_G$ such that for every $\lambda\in\Lambda$, there is a constant $C>0$ (possibly depending on $f,\psi,\Lambda,H,\lambda$) such that for all $g, \widetilde{g} \in G$,
\[
|\Xi(\Psi(g \cdot f))(\lambda)-\Xi(\Psi(\widetilde{g} \cdot f))(\lambda)| \leq C \dist_G(g,\widetilde{g}).
\]
Note that 
\[
|\Xi(\Psi(g \cdot f))(\lambda)-\Xi(\Psi(\widetilde{g} \cdot f))(\lambda)| = |\Xi(\Psi(g\widetilde{g}^{-1} \cdot (\widetilde{g} \cdot f)))(\lambda)-\Xi(\Psi(\widetilde{g} \cdot f))(\lambda)|,
\]
and recall that $\norm{g \cdot f}_{\bL^2(\bR^d)} = \norm{f}_{\bL^2(\bR^d)}$ for all $g \in G$.
Furthermore, since $\dist_G$ is invariant, $\dist_G(g,\widetilde{g}) = \dist_G(g\widetilde{g}^{-1},e)$ for all $g, \widetilde{g} \in G$ (where $e$ denotes the identity element of $G$).
Thus, assuming that $\psi$ is Lipschitz (with respect to the Euclidean metric on its domain $\mathbb{R}^d$) and compactly supported, to prove that transversal pooling is group-Lipschitz, it suffices to show that the right-hand side of~\eqref{eqn: theorem result compact case} has an upper bound of the form $C \norm{f}_{\bL^2(\bR^d)} \dist_G(g,e)$ for some $C>0$ that holds for all $g \in G$ and $f \in \bL^2(\bR^d)$.

The following lemma will allow us to  simplify the right-hand side of~\eqref{eqn: theorem result compact case}.
\begin{lem}\label{lem:dillip}
    Let $g=(A,b) \in G$ and let $\psi$ be Lipschitz with Lipschitz bound $\norm{\psi}_L$. Then $\psi_g$ is Lipschitz with Lipschitz constant
    \[
    \norm{\psi_g}_L \leq \abs{\det{A}}^{-1/2}\norm{A^{-1}}_{2 \rar 2} \norm{\psi}_L.
    \]
\end{lem}

\begin{proof}
For all $x,y\in\bR^d$, it holds that
\begin{align*}
\abs{\psi_g(x)-\psi_g(y)}
&=\abs{\det A}^{-1/2}\abs{\psi(A^{-1}(x-b))-\psi(A^{-1}(y-b))}\\
&\leq \abs{\det A}^{-1/2}\norm{\psi}_L
      \norm{A^{-1}(x-y)}_2
      \leq \abs{\det A}^{-1/2}\norm{\psi}_L
      \norm{A^{-1}}_{2\rar2}\norm{x-y}_2,
\end{align*}
which implies the claim.
\end{proof}

\begin{example}[Rotations]\label{ex: cont rotations}

One of our primary motivations is to develop machine learning methods for images or data cubes with local rotational invariance\footnote{Here, ``local rotational invariance'' means stability to global rotations that are close to the identity. This contrasts with other work, e.g.,~\cite{andrearczyk2020local}, where ``local rotational invariance'' refers to invariance to a full group of rotations that are applied to small patches of an image.}, in which case the relevant symmetry group is~$\SO(2)$.
We will analyze the Lipschitz behavior of $\Xi\circ\Psi$ on  $G=\SO(d)$ for arbitrary $d>1$.

We equip $\SO(d)$ with the following metric.
For $U,\tilde{U}\in\SO(d)$, the real Schur decomposition shows that
$U\tilde{U}^{-1}$ is orthogonally similar to a block diagonal matrix
with rotation blocks
\begin{equation}\label{eqn:schurrot}
\begin{pmatrix}
\cos(\theta_j) & -\sin(\theta_j) \\
\sin(\theta_j) & \cos(\theta_j)
\end{pmatrix},
\qquad
\theta_j\in(-\pi,\pi],
\quad j=1,\ldots,\lfloor d/2\rfloor,
\end{equation}
and one additional block $(1)$ when $d$ is odd.
We define
\begin{equation}\label{eqn:metSOd}
\dist_{\SO(d)}: \SO(d) \times \SO(d) \rar \bR, \quad \dist_{\SO(d)}(U, \tilde{U}) = \bigg(\sum_j |\theta_j|^2\bigg)^{1/2}.
\end{equation}
Up to a constant factor, this is the bi-invariant Riemannian
distance induced by the Frobenius inner product
(e.g.,~\cite{sinclair2008minimum}).

Now, we observe that for any $U \in \SO(d)$, we have $\det(U) = \norm{U}_{2 \rar 2} =1$ and so~\eqref{eqn: CA} gives
\begin{equation}
\label{eqn:Cpurerot}
C_{U} 
= 2\arcsin\left( \frac{1}{2}\norm{U-I}_{2 \rar 2}\right)
=\max_j \abs{\theta_j} 
\leq \bigg(\sum_j |\theta_j|^2\bigg)^{1/2}
=
\dist_{\SO(d)}(U, I).
\end{equation}
Next, we let $\Lambda$ be an arbitrary closed subgroup of $G$, let $H$ be an arbitrary precompact set of coset representatives, and observe that  we have $\tilde{R}_{\lambda h} = R$ for all $\lambda\in\Lambda, h\in H$. 
Also, Lemma~\ref{lem:dillip} implies $\|\psi_{\lambda h}\|_L=\|\psi\|_L$, and
so for $g=(U,0)$, the bound~\eqref{eqn: theorem result compact case} gives
\begin{equation}
|\Xi(\Psi(g \cdot f))(\lambda)-\Xi(\Psi(f))(\lambda)|
\leq \omega_d^{1/2} R^{(d+2)/2} \norm{\psi}_{L} \|f\|_{\bL^2(\mathbb{R}^d)}\cdot\dist_{\SO(d)}(U, I).\label{eqn: rotation bound examle from theorem}
\end{equation}
We see that this bound does not depend on the choice of $\Lambda$, which will not be the case in our next example (Example~\ref{ex: cont trans}).
While this is evidence that the bound may not be tight in general, it does show that for a particular application, one may choose a pooling domain that decreases computational complexity while preserving distinctions relevant to the task without changing this bound.
For example, in Section~\ref{sec: decreasing training set}, we want to distinguish between horizontally dominant and vertically dominant ellipses and thus do not pool rotations of $90^\circ$ together.
\end{example}

\begin{remark}\label{rem: rotation better bound}
In the case where $d=2$, we may further simplify~\eqref{eqn: rotation bound examle from theorem} by noting that we may write a generic element of $\SO(2)$ as 
\[
U = \begin{pmatrix}
    \cos(\theta) & -\sin(\theta) \\
    \sin(\theta) & \cos(\theta)
\end{pmatrix},
\]
with $\theta \in (-\pi,\pi]$.
So, since $\abs{\theta} = \dist_{\SO(2)}(U,I)$, the right-hand side of~\eqref{eqn: rotation bound examle from theorem} becomes
\begin{equation}
\pi^{1/2} R^2 \norm{\psi}_{L} \|f\|_{\bL^2(\mathbb{R}^2)}\abs{\theta}.\label{eqn: 2drotation bound examle from theorem}
\end{equation}
We also note that we may improve on this estimate by using techniques specific to $\SO(2)$ to show that for every $\lambda$ we have 
\begin{align}\label{eqn: rotation specific bound}
|\Xi(\Psi(g\cdot f))(\lambda)-\Xi(\Psi(f))(\lambda)|
\leq  2\pi^{1/2}R^2\|\psi\|_L \sin\left(\frac{|\theta|}{2}\right)\|f\|_{\bL^2(\mathbb{R}^2)}.
\end{align}
For a proof of~\eqref{eqn: rotation specific bound}, please see Appendix~\ref{sec: details on better rotation bound}, where it is stated as Proposition~\ref{prop: Rotation special case}.
Since $\sin(x)\sim x$ as $x\to0$,
the right-hand sides of~\eqref{eqn: 2drotation bound examle from theorem} and~\eqref{eqn: rotation specific bound} agree to leading order.
See Section~\ref{sec: continuity of pooling} for a numerical illustration in the context of a digital implementation of $\Xi\circ\Psi$.
\end{remark}

\begin{example}[Translations]\label{ex: cont trans}
For the translation group $G=\bR^d$, we use the translation-invariant Euclidean metric, i.e., $\text{dist}_{\text{Eucl}}(b,0)=\|b\|_2$.
Fix $c>0$, let $\Lambda$ be a lattice of the form $c\mathbb{Z}^d=\{(cm_1,\ldots,c m_d):m_i\in\mathbb{Z}\}$, and choose $H=[-\frac{c}{2},\frac{c}{2})^d$.
Recall that $C_I=0$.
We also have
\[
\tilde{R}_{\lambda h} = \norm{I}_{2 \rar 2} R + \norm{b_{\lambda h}}_2 \leq R + \norm{b_{\lambda}}_2 +\frac{c}{2}\sqrt{d},
\]
and $\|\psi_{\lambda h}\|_L=\|\psi\|_L$ by  Lemma~\ref{lem:dillip}. Thus, for every $\lambda$, the bound from~\eqref{eqn: theorem result compact case} gives 
\begin{equation}
|\Xi(\Psi((I,b)\cdot f))(\lambda)-\Xi(\Psi(f))(\lambda)|
\leq 
\omega_d^{1/2} \left(R+\|b_\lambda\|_2+\frac{c}{2}\sqrt{d}\right)^{d/2}\|\psi\|_{L}\|f\|_{\bL^2(\mathbb{R}^d)}\cdot\dist_{\operatorname{Eucl}}(b,0).\label{eqn:thm2trans}
\end{equation}
The displayed upper bound decreases as the lattice spacing $c$ tends to zero.
\end{example}

\begin{example}[Isotropic dilations]\label{ex: cont dilations}
The isotropic dilation group $G=\{s I : s>0\}\cong \bR_{>0} = (0,\infty)$ has a standard invariant distance:
\[
\dist_{\bR_{>0}} (s, \tilde{s}) = \abs{\ln(s/\tilde{s})}.
\]
Fix $a>1$,  let $\Lambda$ be the subgroup of $G$ corresponding to powers of $a$, i.e., $\Lambda=\{\lambda_j(x)=a^jx: j\in\mathbb{Z}\}$, and choose $H=[1,a)$.
To interpret the results of Theorem~\ref{thm: continuity of xi} in this setting, we note that  
\[
C_{sI}= \abs{\ln s}, 
\qquad
\tilde{R}_{\lambda_j h} = R a^j h, \qquad
\norm{\psi_{\lambda_j h}}_L \leq (a^j h)^{-(d+2)/2} \norm{\psi}_L,
\]
where $\lambda_j=a^jI$ and $\|\psi_{\lambda_j h}\|_L$ is bounded  via Lemma~\ref{lem:dillip}.  
Thus, \eqref{eqn: theorem result compact case} gives
\[
|\Xi(\Psi(g\cdot f))(\lambda)-\Xi(\Psi(f))(\lambda)|
\leq 
\left(\frac{d}{2} \norm{\psi}_{\bL^2(\bR^d)} + 
R^{(d+2)/2} \omega_d^{1/2} \norm{\psi}_L\right)\|f\|_{\bL^2(\mathbb{R}^d)}\cdot\dist_{\bR_{>0}}(s,1).
\]
As in Example~\ref{ex: cont rotations}, and unlike
Example~\ref{ex: cont trans}, our bound is independent of the
sampling subgroup $\Lambda$, here determined by the dilation factor $a$.
In Example~\ref{ex:dirsim} below, however, combining isotropic
dilations with rotations and translations yields a bound that
depends on $a$.
\end{example}

In the examples thus far, $G\leq\AGL(\bR^d)$ has either been a subgroup of $\GL(\bR^d)$ or of $\bR^d$.
The next two examples have nontrivial components in both $\GL(\bR^2)$ and $\bR^2$.

\begin{ex}[Rototranslations]\label{ex:rot4}
Let $G = \SE(2)=\SO(2) \ltimes \bR^2$.
For $\theta\in (-\pi,\pi]$, we define $U_\theta\in \SO(2)$ by 
\[
U_\theta = \begin{pmatrix} \cos(\theta) & -\sin(\theta) \\ \sin(\theta) & \cos(\theta) \end{pmatrix}
\]
and consider the following right-invariant metric on the special Euclidean group $G$
\begin{equation}\label{eqn:metse2}
\dist_{\SE(2)} ((U_\theta,b),(U_{\tilde{\theta}},\tilde{b})) = \min_{k\in \bZ} \abs{\theta-\tilde{\theta}+2\pi k} +\norm{b-U_{\theta-\tilde{\theta}}\tilde{b}}_2.
\end{equation}
This metric arises by computing $(V,v) = (U_\theta,b)(U_{\tilde{\theta}},\tilde{b})^{-1}$, applying the standard bi-invariant geodesic distance on $\SO(2)$ from $I$ to $V$ to attain the first term and the (bi-invariant) Euclidean distance between $v$ and $0$ to attain the second term.
Let $L = \left\{U_{k \pi/2} \, : \, k \in \{0,1,2,3\} \right\}$ be the subgroup of $\SO(2)$ with four elements, 
and let $\Lambda = L \ltimes c\bZ^2$ for some fixed $c>0$.
Note that $c\bZ^2$ is invariant under multiplication by elements of $L$, i.e., $L(c\bZ^2)= c\bZ^2$,
yielding that $\Lambda\leq G\leq \AGL(\bR^2)$ are nested subgroups.
Let $H=\left\{U_\theta \, : \, \theta \in [0,\pi/2)\right\} \times [-\frac{c}{2},\frac{c}{2})^2$. 

We recall from~\eqref{eqn:Cpurerot} that
$C_{U_\theta} =  \abs{\theta}$ and
similar to Example~\ref{ex: cont trans} note that
\[
\tilde{R}_{\lambda h} = R + \norm{b_{\lambda h}}_2 \leq R+ \norm{b_\lambda}_2 + \frac{c}{\sqrt{2}},
\qquad 
\norm{\psi_{\lambda h}}_L = \norm{\psi}_L.
\]
Plugging this into~\eqref{eqn: theorem result compact case}, we obtain for every $\lambda$ that
\begin{align}
\lefteqn{|\Xi(\Psi((U_\theta,b)\cdot f))(\lambda)-\Xi(\Psi(f))(\lambda)|}\nonumber\\
&\leq \left(\left(R+ \norm{b_\lambda}_2 + \frac{c}{\sqrt{2}} \right) \abs{\theta}+\norm{b}_2 \right)\left(R+ \norm{b_\lambda}_2 + \frac{c}{\sqrt{2}} \right)\pi^{1/2}\norm{\psi}_{L}\|f\|_{\bL^2(\mathbb{R}^2)}\nonumber\\
&\leq\max\left\{R+ \norm{b_\lambda}_2 + \frac{c}{\sqrt{2}},1 \right\} \left(R+ \norm{b_\lambda}_2 + \frac{c}{\sqrt{2}} \right)\pi^{1/2}\norm{\psi}_{L}\|f\|_{\bL^2(\mathbb{R}^2)}
\cdot\dist_{\SE(2)}((U_\theta,b),(I,0)).\label{eqn:bdrot4}
\end{align}
The only part of~\eqref{eqn:bdrot4} that depends on $\Lambda$ is $R+\|b_\lambda\|_2+\frac{c}{\sqrt{2}}$, which decreases as the size of the lattice $c$ tends to $0$.
Interestingly, we obtain the same group-Lipschitz bound on $\Psi$ instead of $\Xi\circ\Psi$ by equivariance and the Cauchy--Schwarz inequality:
\[
\abs{\Psi((U_\theta,b)\cdot f)(\lambda)-\Psi(f)(\lambda)}
\leq \norm{\psi_{(U_\theta,b)^{-1}\lambda}-\psi_{\lambda}}_{\bL^2(\bR^2)} \norm{f}_{\bL^2(\bR^2)},
\]
from which an application of Theorem~\ref{prop:contest} and Lemma~\ref{lem: continuity simple} reproduces the bound~\eqref{eqn:bdrot4}.
Thus, pooling with $\Xi$ has a benign impact on sensitivity to rototranslations while reducing the size of the representation.
\end{ex}

Many invariant neural network architectures following the geometric deep learning blueprint (e.g., \cite{cohen2016group, hong2022multiscale, marcos2016learning,laptev2016ti}) combine pooling over $L\leq\GL(\bR^d)$ with spatial max pooling.
Our next example considers the case where $L$ is the group of rotations of the plane by multiples of $\pi/2$. 
\begin{ex}[Discrete rototranslations]\label{ex:rot4to2}
Let $G = \left\{U_{k \pi/2} \, : \, k \in \{0,1,2,3\} \right\} \ltimes \bZ^2$ (i.e., a discrete subgroup of $\SE(2)$), $\Lambda = \{I\}  \ltimes c\bZ^2$ for $c\in\mathbb{N}$ with $c\geq 2$, and
\[
H = \left\{U_{k \pi/2} \, : \, k \in \{0,1,2,3\}\right\}\times \left\{-\lfloor\tfrac{c}{2}\rfloor,\ldots,\lfloor\tfrac{c-1}{2} \rfloor\right\}^2.
\]
The group $G$, which is also known as $p4$, was studied in the foundational
work on group-equivariant neural networks~\cite{cohen2016group}.
Since $G$ and $\Lambda$ are subgroups of $\SE(2)$, we may reuse calculations from Example~\ref{ex:rot4}:
\begin{align}
\lefteqn{|\Xi(\Psi((U_\theta,b)\cdot f))(\lambda)-\Xi(\Psi(f))(\lambda)|}\nonumber\\
&\leq \left(\left(R+ \norm{b_\lambda}_2 + \frac{c}{\sqrt{2}} \right) \abs{\theta}+\norm{b}_2 \right)\left(R+ \norm{b_\lambda}_2 + \frac{c}{\sqrt{2}} \right)\pi^{1/2}\norm{\psi}_{L}\|f\|_{\bL^2(\mathbb{R}^2)}\nonumber\\
&\leq\max\left\{R+ \norm{b_\lambda}_2 + \frac{c}{\sqrt{2}},1 \right\} \left(R+ \norm{b_\lambda}_2 + \frac{c}{\sqrt{2}} \right)\pi^{1/2}\norm{\psi}_{L}\|f\|_{\bL^2(\mathbb{R}^d)}
\cdot \dist_{\SE(2)}((U_\theta,b),(I,0)).\nonumber
\end{align}
This bound is not tight in all situations, since, e.g., for $c$ odd and $\lambda = (I,0)$, we have
\[
|\Xi(\Psi((U_{\pi},0)\cdot f))(\lambda)-\Xi(\Psi(f))(\lambda)| = 0,
\]
while $\dist_{\SE(2)}((U_{\pi},0),(I,0)) = \pi$.
\end{ex}

\subsection{Group-Lipschitz Continuity of Affine Actions Which Do Not Form a Group}\label{sec:affine4realz}

In Section~\ref{sec: Group lipschitz pooling}, the retained index set $\Lambda$ takes the form $\Lambda_D\ltimes\Lambda_T$, meaning each element of $\Lambda_D\leq\GL(\bR^d)$ fixes $\Lambda_T\leq\bR^d$ as a set.
This compatibility condition ensures that $\Lambda_D\times\Lambda_T$ is a subgroup under affine composition, but this is overly restrictive.
For example, the dyadic dilation group $\{2^jI:j\in\bZ\}$ does not preserve the integer lattice $\bZ^d$, even though existing wavelets and neural networks simultaneously accommodate both groups (as well as other combinations of dilation and translation groups)~\cite{cohen2016group,romero2020wavelet,
dieleman2016exploiting,fasel2006rotation,sifre2013rotation,
ver2023using}.
This motivates an extension of our theory that allows one to choose the sampling subgroups $\Lambda_D\leq\GL(\bR^d)$ and $\Lambda_T\leq\bR^d$ independently of each other.

In this section, we assume $G = L \ltimes \bR^d$ for some closed $L \leq \GL(\bR^d)$.
Instead of choosing a single subgroup $\Lambda$ of $G$, we choose a closed dilation subgroup $\Lambda_D \leq L$ and a closed translation subgroup $\Lambda_T \leq \bR^d$.
Importantly, we do not insist that $\Lambda_D \times \Lambda_T$ forms a subgroup of $\AGL(\bR^d)$.
We further assume that $L/\Lambda_D$ and $\bR^d/\Lambda_T$ are compact and choose $H_D$ as a precompact set of right coset representatives of $\Lambda_D$ in $L$ and similarly select a precompact set $H_T$ of right coset representatives of $\Lambda_T$ in $\bR^d$.
Since $\Lambda_D \times \Lambda_T$ is not necessarily a subgroup, $H_D \times H_T$ need not be a set of coset representatives, but it is a \emph{transversal}, meaning every element of $G$ equals $(\lambda_Dh_D,\lambda_T+h_T)$ for a unique choice of $(\lambda_D,\lambda_T)\in\Lambda_D\times\Lambda_T$ and $(h_D,h_T)\in H_D\times H_T$.

Now that we have a choice of transversal, we define the corresponding transversal pooling operator $\Xi_w: \mathbf{C}_b(G)\rar\bC_b(\Lambda_D \times \Lambda_T)$ by
 \begin{equation}
     \label{eqn: pooling2}
     \Xi_w(F)(\lambda_D,\lambda_T) = \sup_{(h_D,h_T)\in H_D \times H_T} F(\lambda_D h_D,\lambda_T+ h_T).
 \end{equation}
Next, we observe that $\Xi_w$ exhibits some level of group equivariance.
Let $\Lambda_T$ act on $\bC_b(\Lambda_D \times \Lambda_T)$ by
 \begin{equation*}
({\lambda}_T \cdot F)(\tilde\lambda_D,\tilde\lambda_T) 
= F(\tilde\lambda_D, \tilde\lambda_T-{\lambda}_T)
\end{equation*} 
for every $(\tilde\lambda_D,\tilde\lambda_T)\in \Lambda_D \times \Lambda_T$ and ${\lambda}_T \in \Lambda_T$.
Then the following result may be proved by repeating the proof of Theorem~\ref{prop: xi equivariance}
mutatis mutandis.

\begin{prop}\label{prop: xi equivariance2}
$\Xi_w$ as defined in~\eqref{eqn: pooling2} is $\Lambda_T$-equivariant from $\mathbf{C}_b(G)$ to $\bC_b(\Lambda_D \times \Lambda_T)$, i.e., for all $F\in\bC_b(G)$ and $\lambda_T\in\Lambda_T$, we have
\[
\Xi_w(\lambda_T\cdot F)={\lambda_T}\cdot \Xi_w(F).
\]    
\end{prop}

Recall from Remark~\ref{rem: Proof of Lemma Continuity Simple} that the proof of Lemma~\ref{lem: continuity simple} does not use the $\Lambda$-equivariance of~$\Xi$. 
Therefore, a result analogous to Lemma~\ref{lem: continuity simple} holds with $\Xi_w$ in place of $\Xi$. 
Furthermore, the subgroup $\Lambda$ appears in neither the proof nor the statement of Theorem~\ref{prop:contest}. 
Thus, the proof of Theorem~\ref{thm: continuity of xi} also gives the following result.

\begin{prop}[Group-Lipschitz Continuity of $\Xi_w\circ \Psi$]\label{prop: continuity of xiw}
Let $g=(A,b)\in G$, and let $\psi \in \bW^{1,2}(\bR^d)$. Assume $\det(A) >0$ and that  $\big{\|}\norm{\cdot}_2 \nabla \psi\big{\|}_{\bL^2(\bR^d)}< \infty$. 
Then for every $f\in\bL^2(\mathbb{R}^d)$ and $(\lambda_D,\lambda_T)\in\Lambda_D \times \Lambda_T$, we have 
\begin{align*}
\lefteqn{|\Xi_w(\Psi(g\cdot f))(\lambda_D,\lambda_T)-\Xi_w(\Psi(f))(\lambda_D,\lambda_T)|}\\
&\leq \sup_{(h_D,h_T)\in H_D \times H_T}
\bigg( C_{A} \big{\|}\norm{\cdot}_2 \nabla \psi_{(\lambda_D h_D,\lambda_T+h_T)}\big{\|}_{\bL^2(\bR^d)}
+\frac{1}{2}\abs{\ln \det A^{}} \norm{\psi_{(\lambda_D h_D,\lambda_T+h_T)}}_{\bL^2(\bR^d)} \\
&\hspace{3.2in}  
+\norm{A^{-1}b}_2 \norm{\nabla \psi_{(\lambda_D h_D,\lambda_T+h_T)}}_{\bL^2(\bR^d)}\bigg)\cdot\|f\|_{\bL^2(\mathbb{R}^d)},
\end{align*}
where $C_A$ is defined as in~\eqref{eqn: CA}.
If we further assume that $\psi$ is Lipschitz and compactly supported with Lipschitz constant $\norm{\psi}_{L}$ and $\supp(\psi) \subseteq B(0,R)$, then for every $f\in\mathbf{L}^2(\mathbb{R}^d)$  and $(\lambda_D,\lambda_T)\in\Lambda_D \times \Lambda_T$, we have
\begin{align}
\lefteqn{|\Xi_w(\Psi(g\cdot f))(\lambda_D,\lambda_T)-\Xi_w(\Psi(f))(\lambda_D,\lambda_T)|}\nonumber\\
&\leq \bigg( \sup_{(h_D,h_T) \in H_D \times H_T} \left( \tilde{R}_{(\lambda_D h_D,\lambda_T+h_T)} C_{A}+\norm{A^{-1}b}_2\right) \tilde{R}_{(\lambda_D h_D,\lambda_T+h_T)}^{d/2}\omega_d^{1/2}\norm{\psi_{(\lambda_D h_D,\lambda_T+h_T)}}_{L}\nonumber \\
&\hspace{3.8in} + \frac{1}{2}\abs{\ln \det A}\norm{\psi}_{\bL^2(\bR^d)}\bigg)\cdot\|f\|_{\bL^2(\mathbb{R}^d)},
\label{eqn: theorem result compact case w}
\end{align}
where $\omega_d$ is the volume of the unit ball in $\mathbb{R}^d$ and
$\tilde{R}_{(\lambda_D h_D,\lambda_T+h_T)} = R\norm{\lambda_D h_D}_{2 \rar 2} + \norm{\lambda_T+h_T}_2$.
\end{prop}

We now turn our attention to the  direct similitude group $(\bR_{>0} \times \SO(d)) \ltimes \bR^d$ and  would like to show that $\Xi_w\circ\Psi$ is group-Lipschitz. 
To retain the explicit rotation, translation, and logarithmic scaling terms in our estimates, we use the following right-invariant semimetric.\footnote{A \emph{semimetric} is a non-negatively valued, positive definite, symmetric, bivariate map (see, e.g., \cite{gottlieb2017nearly}).
In other words, a semimetric is a metric, except it possibly violates the triangle inequality.}

\begin{lem}\label{lem:similmet}
Let $G$ be the direct similitude group $(\bR_{>0} \times \SO(d)) \ltimes \bR^d$.
Define $N: G \rar \bR_{\geq 0}$ as
\[
N((s,U,b)) = \abs{\ln s} + \dist_{\SO(d)}(U,I) + s^{-1}\norm{b}_2,
\]
where $(s,U,b)$ stands for the element of $G$ mapping $x \mapsto s U x + b$ and $\dist_{\SO(d)}$ is the bi-invariant Riemannian distance~\eqref{eqn:metSOd}.
Then, 
\[\disim_{(\bR_{>0} \times \SO(d)) \ltimes \bR^d} : \left((\bR_{>0} \times \SO(d)) \ltimes \bR^d\right) \times \left((\bR_{>0} \times \SO(d)) \ltimes \bR^d\right) \rar \bR_{\geq 0}\]
defined as
\begin{align}
\disim_{(\bR_{>0} \times \SO(d)) \ltimes \bR^d}\left( (s,U,b), (\tilde{s},\tilde{U},\tilde{b})\right)
&= \min\{ N((s,U,b)(\tilde{s},\tilde{U},\tilde{b})^{-1}),\, N((\tilde{s},\tilde{U},\tilde{b})(s,U,b)^{-1})\}\label{eqn:simdisim}
\end{align}
is positive definite, symmetric, and right-invariant.
\end{lem}

\begin{proof}
    The definition of $\disim_{(\bR_{>0} \times \SO(d)) \ltimes \bR^d}$ immediately yields that it is non-negatively valued, symmetric, and right-invariant.
    Next, $N((s,U,b))=0$ precisely when all three constituent terms are zero,
    meaning $(s,U,b)=(1,I,0)$.
    It follows that $\disim_{(\bR_{>0} \times \SO(d)) \ltimes \bR^d}$ is positive definite.
\end{proof}

\begin{ex}[Direct similitudes]
\label{ex:dirsim} Let $G$ be the direct similitude group $(\bR_{>0} \times \SO(d)) \ltimes \bR^d$ that acts on $\bR^d$ as $x \mapsto s U x + b$.
In order to specify $\Xi_w$, fix $a>1$, $c >0$, and a closed subgroup $\Lambda_r \leq \SO(d)$, where $H_r$ is a precompact set of right coset representatives of $\Lambda_r$ in $\SO(d)$.
Then define $\Lambda$ to be the set
\[
\Lambda=\{(a^\ell, W, ck) \,:\,\ell \in \bZ,\, W \in \Lambda_r, \,k\in\mathbb{Z}^d\}
\]
and $H$ the transversal $H = [1,a) \times H_r \times \left[-\frac{c}{2},\frac{c}{2} \right)^d$.
We will use Proposition~\ref{prop: continuity of xiw} to show that $\Xi_w\circ\Psi$ is group-Lipschitz with respect to the semimetric defined in Lemma~\ref{lem:similmet}.
To that end, \eqref{eqn:Cpurerot} gives
\[
C_{s U} \leq \abs{\ln s} + \dist_{\SO(d)}(U,I)
\]
for every $s > 0$ and $U \in \SO(d)$.
Next, for $(a^\ell,W,ck)\in\Lambda$, we write $\lambda_D:=a^\ell W$ and $\lambda_T:=ck$, and for $h=(h_s,h_r,h_T)\in H$, we write $h_D:=h_s h_r$.
Then
\[
\tilde{R}_{(\lambda_D h_D,\lambda_T+h_T)} 
= R\norm{\lambda_D h_D}_{2 \rar 2} + \norm{\lambda_T+h_T}_2
= a^\ell h_s R + \norm{ck +h_T}_2
\leq a^\ell h_s R +c\norm{k}_2 + \frac{c}{2} \sqrt{d},
\]
while Lemma~\ref{lem:dillip} gives
\[
\norm{\psi_{(\lambda_D h_D,\lambda_T+h_T)}}_{L} 
\leq (a^\ell h_s)^{-d/2} (a^\ell h_s)^{-1} \norm{\psi}_L 
=(a^\ell h_s)^{-(d+2)/2}  \norm{\psi}_L.
\]
Finally, for each $(s,U,b)\in G$, we have $\norm{(sU)^{-1}b}_2 = s^{-1} \norm{b}_2$ and $\abs{\ln \det sU} = d \abs{\ln s}$.
Putting this all together, one may obtain from~\eqref{eqn: theorem result compact case w} that
\begin{equation}
\label{eqn:onesidesimil}
|\Xi_w(\Psi((sU, b)\cdot f))(\lambda_D,\lambda_T)-\Xi_w(\Psi(f))(\lambda_D,\lambda_T)|
\leq C \cdot N(s,U,b) \cdot \|f\|_{\bL^2(\mathbb{R}^d)}
\end{equation}
for every $f\in \bL^2(\mathbb{R}^d)$ and $(s,U,b)\in G$,
where
\begin{align}
C &:= \max \bigg\{ \omega_d^{1/2} \norm{\psi}_L \left( R + \frac{c}{a^\ell}\left(\norm{k}_2 + \sqrt{d}/2\right)\right)^{(d+2)/2} + \frac{d}{2}\norm{\psi}_{\bL^2(\bR^d)},\nonumber\\
&\hspace{2.5in}\omega_d^{1/2} \norm{\psi}_L(a^\ell)^{-1}\left( R + \frac{c}{a^\ell }\left(\norm{k}_2 + \sqrt{d}/2\right)\right)^{d/2} \bigg\}.\label{eqn:Cbdsimil}
\end{align}
Finally, we upgrade the bound~\eqref{eqn:onesidesimil} using the facts that 
\[
|\Xi_w(\Psi(g\cdot f))(\lambda_D,\lambda_T)-\Xi_w(\Psi(f))(\lambda_D,\lambda_T)|
= 
|\Xi_w(\Psi(g^{-1} \cdot (g\cdot f)))(\lambda_D,\lambda_T)-\Xi_w(\Psi(g \cdot f))(\lambda_D,\lambda_T)|
\]
and $\|g \cdot f\|_{\bL^2(\mathbb{R}^d)}
=\|f\|_{\bL^2(\mathbb{R}^d)}$:
\[
|\Xi_w(\Psi(g\cdot f))(\lambda_D,\lambda_T)-\Xi_w(\Psi(f))(\lambda_D,\lambda_T)|
\leq C \cdot \|f\|_{\bL^2(\mathbb{R}^d)}\cdot \disim_{(\bR_{>0} \times \SO(d)) \ltimes \bR^d}\left( (s,U,b), (1,I,0)\right).
\]
That is, $\Xi_w\circ\Psi$ is group-Lipschitz, as desired.
Analyzing~\eqref{eqn:Cbdsimil}, we see that $C$ is smaller when $c$ is smaller, i.e., when we pool over a smaller region of translations.  
This is in line with what we saw in Example~\ref{ex: cont trans}.
Meanwhile, unlike Example~\ref{ex: cont dilations}, the bound varies with $a$ in a way that depends on the sign of $\ell$.
\end{ex}

\section{Numerical Results}\label{sec: experiments}

In this section, we illustrate our stability estimates and evaluate \ournet{}s on real and synthetic data.\footnote{The code used to conduct all the experiments in this section can be found at \url{https://github.com/zyjux/trapnets}.}
Specifically, Section~\ref{sec: continuity of pooling} examines how the output of a digital implementation of $\Xi\circ\Psi$ changes under rotation, providing a qualitative illustration of our stability estimates.
Next, we demonstrate the performance of our method in data-limited environments in Section~\ref{sec: decreasing training set}.
Finally, we apply \ournet{}s to tropical cyclone data in Section~\ref{sec: experiments real world}. 
See Appendix~\ref{app: implementation} for implementation details.

\subsection{Group-Lipschitz Continuity of Transversal Pooling over a Digital Domain}
\label{sec: continuity of pooling}

\begin{figure}
    \centering
    \includegraphics[width=0.45\linewidth]{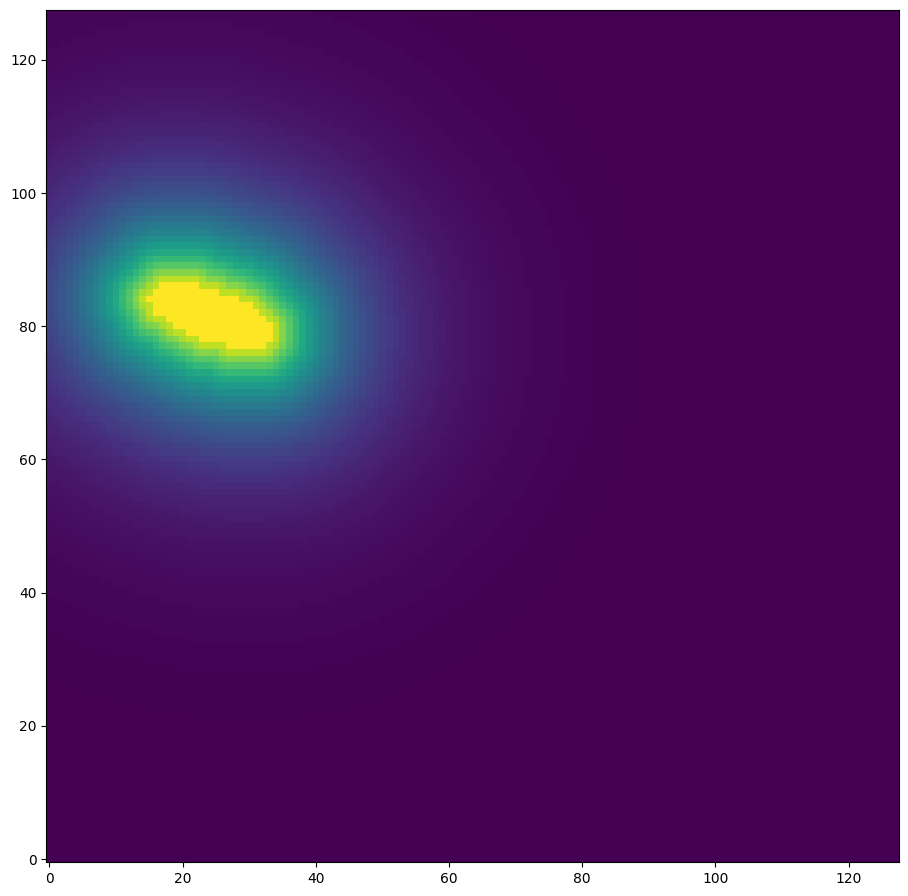}
    \qquad\quad
    \includegraphics[width=0.45\textwidth]{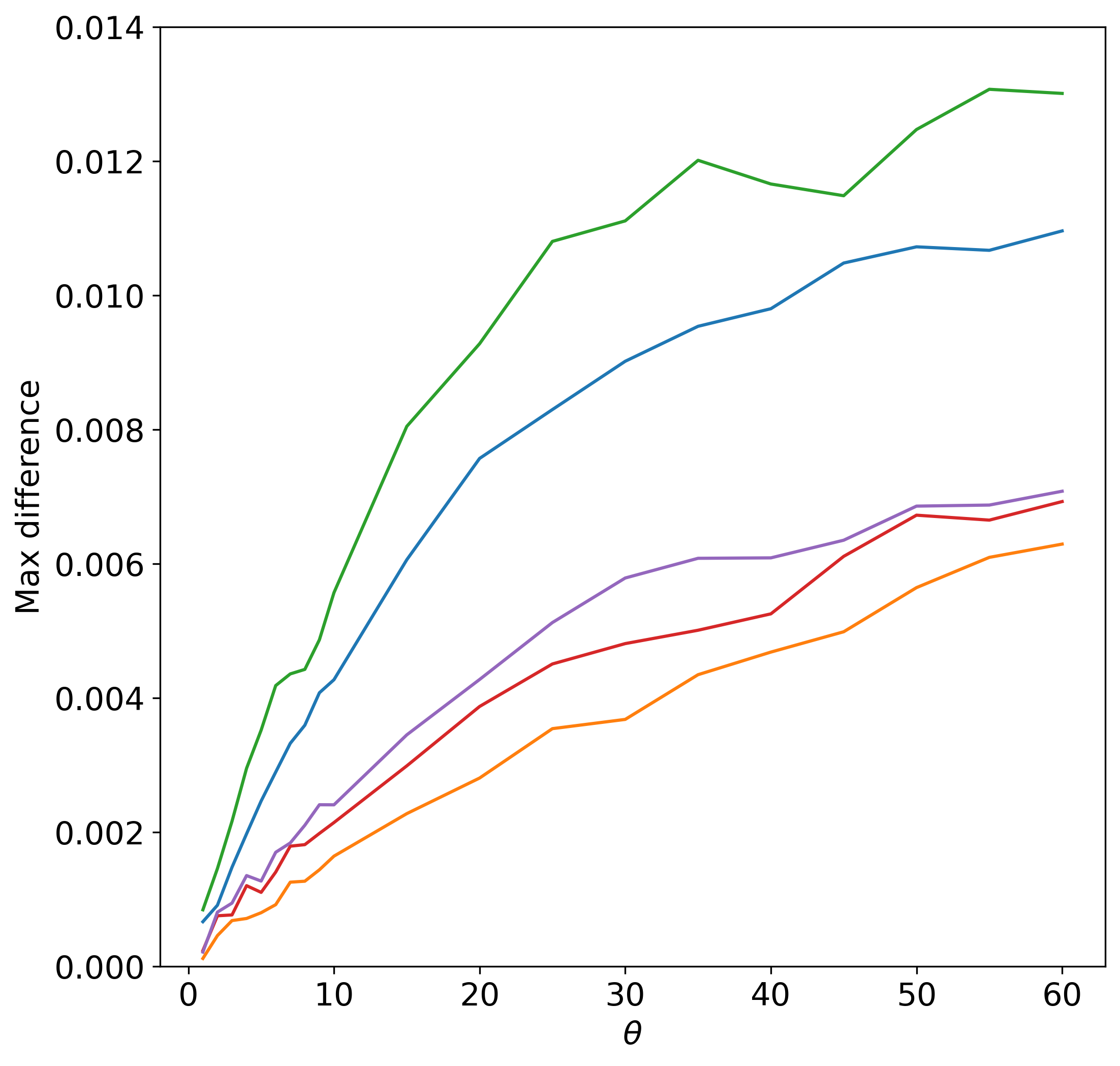}
    \caption{Left: A randomly generated  ellipse image used for our experiments in Section~\ref{sec: continuity of pooling}. Right: The maximum over $\lambda\in\Lambda$ of the absolute difference between the pooled wavelet coefficients of an image and its rotated version, divided by the $L^2$ norm of the unrotated image. The rotation angle $\theta$ is measured in degrees, and each colored curve corresponds to a different random trial.}
    \label{fig: example synthetic ellipse}
    \label{fig: continuity pooling vs not}
\end{figure}

Motivated by the stability estimate in Remark~\ref{rem: rotation better bound}, we first examine how a digital implementation of $\Xi\circ\Psi$ responds to two-dimensional rotations.
For $f\ne0$, dividing~\eqref{eqn: rotation specific bound} by $\|f\|_{\bL^2(\mathbb R^2)}$ gives
\begin{equation}
\label{eqn: Rotation continuity bound fraction}
\frac{|\Xi(\Psi((U_\theta,0)\cdot f))(\lambda)-\Xi(\Psi(f))(\lambda)|}{\|f\|_{\bL^2(\mathbb{R}^2)}}
\leq
2\sqrt{\pi}R^2\|\psi\|_{L}
\sin\left(\frac{|\theta|}{2}\right).
\end{equation}
For each $\theta$, we compute a digital analogue of the left-hand side of~\eqref{eqn: Rotation continuity bound fraction} and maximize over $\lambda\in\Lambda$.
Since digital rotations require resampling, this experiment provides a qualitative illustration of local stability.

Of course, we can only compute finitely many rotation angles.
Thankfully, the proof of~\eqref{eqn: rotation specific bound} also applies to finite subgroups of $\operatorname{SO}(2)$.
We choose $G\cong C_{360}$ to be the subgroup generated by the counterclockwise rotation by $1^\circ$.
For the transversal pooling operator $\Xi$, we choose the subgroup $\Lambda\cong C_{72}$ generated by a $5^\circ$ rotation, and we take $H$ to consist of rotations by $0^\circ,1^\circ,2^\circ,3^\circ,$ and $4^\circ$.
The filters are drawn at random and not trained; see Appendix~\ref{subsection.section 4.1} for implementation details.

For the input signals $f$, we use synthetic $128\times128$ images containing an ellipse $E$ of random size, orientation, and location, as illustrated in Figure~\ref{fig: example synthetic ellipse}.
The major-axis length is drawn uniformly from $[10,40)$ pixels, and the minor-axis length is drawn uniformly from the interval between 5 pixels and two-thirds of the major-axis length.
The major-axis orientation is drawn uniformly from $[0,\pi)$, and the center is drawn uniformly from $[20,108)^2$.
We set pixel intensities to 1 inside the ellipse and let them decay exponentially outside it, assigning each exterior pixel $x$ the intensity $\exp(-\operatorname{dist}_2(x,E)/10)$, where $\operatorname{dist}_2(x,E)$ denotes the Euclidean distance from $x$ to $E$.

We test counterclockwise rotations by $1^\circ,\ldots,10^\circ$ and $15^\circ,20^\circ,\ldots,60^\circ$.
Figure~\ref{fig: continuity pooling vs not} plots the maximum normalized difference as a function of $\theta$, with each curve corresponding to a different random trial.
For small rotation angles, the observed differences grow approximately linearly, consistent with the local behavior of the bound in~\eqref{eqn: Rotation continuity bound fraction}.

\subsection{Effects of Decreasing Training Set Size}
\label{sec: decreasing training set}

\begin{table}[t]
    \centering
    \begin{tabular}{c|cccccccc}
         Training Set Size & 8000 & 4000 & 2000 & 800 & 400 & 80 & 40 & 16 \\
         \hline
         CNN Acc. & 98.4 & 97.5 & 98.6 & 97.6 & 48.9 & 51.1 & 51.1 & 49.7 \\
         Aug. CNN Acc. & 98.5 & \textbf{98.9} & 98.1 & 96.6 & 48.9 & 51.1 & 59.5 & 51.2 \\
         \ournet{} Acc. & \textbf{98.7} & 98.5 & \textbf{98.9} & \textbf{98.4} & \textbf{97.2} & \textbf{75.5} & \textbf{65.7} & \textbf{51.7}
    \end{tabular}
    \caption{The accuracy of \ournet{} models compared to baseline CNNs and augmented CNNs when trained on limited data for the orientation task described in Section~\ref{sec: decreasing training set}. 
    The best results appear in bold. 
    Since this is a binary task, $50\%$ accuracy is as good as random guessing.}
    \label{tab: dataset reduction accuracies}
\end{table}

Next, we examine how classification accuracy changes with limited training data, a setting in which equivariant architectures can offer advantages~\cite{bekkers2019b,finzi2020generalizing,weiler2019general}.
We compile 10,000 random ellipse images described in Section~\ref{sec: continuity of pooling}, and our task is to classify each ellipse as horizontally or vertically dominant.
An ellipse is horizontally dominant if the orientation of its major axis relative to the $x$-axis lies in $[-45^\circ,45^\circ)$ modulo $180^\circ$, and vertically dominant otherwise.

A rotation by $90^\circ$ exchanges the two classes, so the classifier must retain orientation information.
We use the \ournet{} described in Section~\ref{sec: equivariant neural network architecturs}, with the filtering operator $\Psi$ defined using the group $G\cong C_{12}\ltimes\mathbb R^2$ generated by $30^\circ$ rotations and all translations.
In the continuous model, the feature extractor that precedes the flattening layer is equivariant to the retained subgroup.\footnote{The digital implementation approximates these operations with resampling and cropping that introduce equivariance errors.}
The subsequent flattening layer and dense classifier allow the output to depend on orientation.

We compare three models: the \ournet{}, a conventional CNN, and the same CNN trained with rotational data augmentation.
All three have eight main convolutional layers, use Leaky-ReLU activations, and employ a dense network with one hidden layer as the final classifier.
To keep the total number of trainable parameters comparable, the \ournet{} includes an additional block of two $1\times1$ convolutional layers before the final rotational pooling and flattening operations.
See Appendix~\ref{subsec.C2} for implementation details and Figure~\ref{fig: synthetic model architectures} for an illustration of the architectures. 

For the augmented CNN, we add two rotated copies of each training image, with rotations of at most $60^\circ$ in either direction.
Each copy retains the original image's class label.
Since these rotations can change the orientation class, this procedure introduces label noise (cf.~\cite{shorten2019survey}).

We divide the dataset into 8000 training images, 1000 validation images, and 1000 test images.
For each model, we train on subsets of the training set containing 8000, 4000, 2000, 800, 400, 80, 40, and 16 images, keeping the validation and test sets fixed.
We use early stopping with a patience of 20 epochs and retain the checkpoint with the best validation performance.

Table~\ref{tab: dataset reduction accuracies} reports test accuracy for each model and training-set size.
When there are at least 2000 training images, all three models exhibit similarly good accuracy.
For 400 training images, the \ournet{} achieves 97.2\% accuracy, while both CNN baselines are comparable to random guessing.
For even smaller training sets, the accuracy of the \ournet{} degrades gracefully.

\bigskip

\subsection{Real-World Tropical Cyclone Data}\label{sec: experiments real world}

Finally, we evaluate how well a \ournet{} predicts rapid intensification (RI) in a given tropical cyclone (TC).
RI occurs when the maximum sustained wind speed of a TC increases by at least 30~knots within a 24-hour period.
For this task, we use the dataset and preprocessing procedures described in~\cite{lagerquist2025identifying}.
Each input is a $380\times540$ infrared satellite image centered at the TC and rotated so that its motion points along the positive $x$-axis.
The classes are strongly imbalanced, with RI occurring in only 5\% of the test set scenes (and in 6\% of the validation set scenes).

As discussed in the introduction, storm motion introduces asymmetries that make orientation informative.
This suggests using a \ournet{} that is stable to small rotations and sensitive to larger ones.
As in Section~\ref{sec: decreasing training set}, we use a \ournet{} with the group $G\cong C_{12}\ltimes\mathbb R^2$ generated by $30^\circ$~rotations and all translations.
We introduce an additional $1\times1$ convolutional block so that the total number of trainable parameters is comparable to that of our baseline, namely, an augmented CNN based on the architecture in~\cite{lagerquist2025identifying}.
See Appendix~\ref{subsec.real world experiment} for implementation details and Figure~\ref{fig: real world architectures} for an illustration of the architectures. 

To address the class imbalance, we train both models using weighted random sampling that targets an RI sampling proportion of~20\%.
We also apply weighted sampling to the validation set used for early stopping.
Both models receive translation and noise augmentation, while the CNN additionally receives rotation augmentation.
The augmentation parameters and remaining training details are given in Appendix~\ref{subsec.real world experiment}.

See Table~\ref{tab: real world results} for the results.
Due to severe class imbalance, with the vast majority of TCs belonging to the non-RI class, a classifier could achieve high-accuracy using the trivial model that always predicts non-RI.
Indeed, such a trivial model would achieve $95$\% accuracy, superior to both of the trained neural networks.
Therefore, in addition to accuracy, we also report the $F_1$ score, precision, and recall for each model, which are likely more informative performance metrics in this context. 
We see that while the augmented CNN achieves slightly higher accuracy, \ournet{} exhibits superior performance in terms of the $F_1$ score, precision, and recall, though both models still miss most RI events.

\begin{table}[]
    \centering
    \begin{tabular}{c|cccc}
         & Accuracy & $F_1$ score & Precision & Recall \\
         \hline
         Aug. CNN & \textbf{0.929} 
         & 0.264 & 0.278 & 0.252 \\
         \ournet{} & 0.926
         & \textbf{0.328} & \textbf{0.305} & \textbf{0.354}
    \end{tabular}
    \caption{Test set results for predicting tropical cyclone rapid intensification. The best results appear in bold.}
    \label{tab: real world results}
\end{table}

\section{Conclusion}

In this paper, we introduced \ournet{}s, defined with respect to $\Lambda \leq G \leq \AGL(\bR^d)$.
Such neural networks have a number of desirable properties: 
\begin{itemize}
\item
they are equivariant to $\Lambda$, 
\item
they are group-Lipschitz with respect to $G$, 
\item
they generalize well in data-limited environments, and
\item
they perform well in real-world settings.
\end{itemize}
The theory used to establish group-Lipschitz continuity also applies to existing architectures in the literature, such as the standard max pooling in CNNs and the full group-invariant pooling in group-equivariant networks.
Along the way, we proved a new bound on the stability of the quasiregular representation of affine subgroups on sufficiently nice functions in $\bL^2(\bR^d)$, which may be of independent interest to the wavelet community.

In closing, we highlight a few opportunities for future work.
For example, it would be interesting to tighten the bound in Theorem~\ref{thm: continuity of xi} (and its generalization in Proposition~\ref{prop: continuity of xiw}).
In particular, the current bound is provably not tight in Example~\ref{ex:rot4to2}.
One might also expand the theory and applications to other group representations, such as the Schr\"odinger representation of the Heisenberg group.
Additionally, the stability of scattering transforms to small
deformations~\cite{sifre2013rotation} motivates extending the results
of Section~\ref{sec:transtheory} to non-affine deformations and establishing corresponding
stability bounds for complete \ournet{}s.
Finally, we believe that there are further interesting applications which only require local approximate invariance and thus would be appropriate candidates for transversal pooling.
For example, perhaps \ournet{}s could be hybridized with a long short-term memory network to forecast hurricanes based on video.

\section*{Acknowledgments, Data Statement, and Conflict of Interest Statement}
EJK and LV were supported in part by NSF CAIG grant RISE-2425923. 
DGM was supported in part by NSF grant DMS 2220304. 
MP was supported in part by NSF grant OIA-2242769.
DGM thanks OpenAI for complimentary access to ChatGPT-6 Pro.
The authors used ChatGPT to find relevant literature, to help generate the figures, and to give feedback on preliminary versions of this manuscript. 
While using ChatGPT to search for literature on stability to affine deformations, EJK discovered the key ideas of the proofs of Lemma~\ref{lem: semigroup bound} and Theorem~\ref{prop:contest} and incorporated them into the paper.
All mathematical claims, citations, and computations were checked by the authors, who take full responsibility for the content of this paper.
All code used for experiments in this paper can be found at \url{https://github.com/zyjux/trapnets}. 
For more information on the data availability of the TC data, see \cite{lagerquist2025identifying}.
On behalf of all authors, the corresponding author states that there is no conflict of interest.

%
%\bibliography{main}{}
%\bibliographystyle{amsalpha}

\newcommand{\etalchar}[1]{$^{#1}$}
\providecommand{\bysame}{\leavevmode\hbox to3em{\hrulefill}\thinspace}
\providecommand{\MR}{\relax\ifhmode\unskip\space\fi MR }
\providecommand{\MRhref}[2]{%
  \href{http://www.ams.org/mathscinet-getitem?mr=#1}{#2}
}
\providecommand{\href}[2]{#2}

\appendix
\section{Lipschitz Continuity under $\SO(2)$ Rotations}
\label{sec: details on better rotation bound}
As alluded to in Remark~\ref{rem: rotation better bound}, the  bounds obtained in Example~\ref{ex: cont rotations} can be improved by a direct estimate in the case where $G$ is the 2D rotation group $\SO(2)$. 
Specifically, we have the 
following proposition.

\begin{prop}\label{prop: Rotation special case}
Let $\psi\in\mathbf{L}^2(\mathbb{R}^2)$ be a Lipschitz function with compact support such that $\supp(\psi)\subseteq B(0,R).$ Let $G=\SO(2)$, let $g\in G$, and as in Example~\ref{ex: cont rotations}, note that there exists $\theta \in (-\pi,\pi]$ such that $g=(U,0)$, where 
$$
U= \begin{pmatrix}
    \cos(\theta) & -\sin(\theta) \\
    \sin(\theta) & \cos(\theta)
\end{pmatrix}.
$$
 Then for every $f\in \bL^2(\mathbb R^2)$ and $\lambda\in\Lambda$,
\begin{align*}
|\Xi(\Psi(g\cdot f))(\lambda)-\Xi(\Psi(f))(\lambda)|
\leq  2\pi^{1/2}R^2\|\psi\|_L \sin\left(\frac{|\theta|}{2}\right)\|f\|_{\bL^2(\mathbb{R}^2)}.
\end{align*}
\end{prop}

The proof of Proposition~\ref{prop: Rotation special case} relies on the following lemma.

\begin{lem}\label{lem: ell infinity diff for different group members}
Assume that $\psi\in\mathbf{L}^2(\mathbb{R}^2)$ is Lipschitz, with Lipschitz constant $\|\psi\|_{L}$, and is compactly supported with $\operatorname{supp}(\psi)\subseteq B(0,R)$.
Let $G=\SO(2)$, and let $g_{1},g_{2}\in G$ with $g_{j}(x)=U_jx$,  
$$
U_j= \begin{pmatrix}
    \cos(\theta_j) & -\sin(\theta_j) \\
    \sin(\theta_j) & \cos(\theta_j)
\end{pmatrix},
$$
for $j=1,2$. 
Then, 
\begin{align*}
 \|\psi_{g_{1}}-\psi_{g_{2}}\|_{\bL^2(\mathbb{R}^2)}\leq 2\pi^{1/2}R^2\|\psi\|_L \abs{\sin\left(\frac{\theta_1-\theta_2}{2}\right)}.
 \end{align*}
\end{lem}

\begin{proof}

 Since $\det(U_j)=1$, we have $\psi_{g_j}(x)=\psi(U_j^{-1}x)$ for $j=1,2$.
Therefore, since  $\operatorname{supp}(\psi_{g_1}-\psi_{g_2})\subseteq B(0,R)$, we may use the fact that $U_1$ is an orthogonal matrix to obtain
\begin{align}
    {\|\psi_{g_{1}}-\psi_{g_{2}}\|_{\mathbf{L}^2(\mathbb{R}^2)}} &=\left(\int_{B(0,R)}{|\psi(U_1^{-1}x)-\psi(U_2^{-1}x)|^2}\,dx\right)^{1/2}\nonumber\\
    &\leq |B(0,R)|^{1/2}\|\psi\|_L\|U_1^{-1}-U_2^{-1}\|_{2\rar 2} R\nonumber\\
    &=\pi^{1/2}R^2\|\psi\|_L \|U_1(U_1^{-1}-U_2^{-1})\|_{2\rar 2}\nonumber\\
              &=\pi^{1/2}R^2\|\psi\|_L \|I-U_1U_2^{-1}\|_{2\rar 2}.\nonumber
    \end{align}
The proof is now completed by observing
that 
$$
U_1U_2^{-1}=\begin{pmatrix}
\cos(\theta_1-\theta_2)&&-\sin(\theta_1-\theta_2)\\
\sin(\theta_1-\theta_2)&&\cos(\theta_1-\theta_2)
\end{pmatrix},
$$
the matrix which corresponds to a rotation by $\theta_1-\theta_2$. Thus,
we have  
$$
\|I-U_1U_2^{-1}\|_{2\rar 2}=|1-e^{i(\theta_1-\theta_2)}|=2\abs{\sin\left(\frac{\theta_1-\theta_2}{2}\right)}
$$
which completes the proof.
\end{proof}

\begin{proof}[Proof of Proposition~\ref{prop: Rotation special case}]
For $g\in G$, $\lambda\in\Lambda$, $h\in H$, choose $\theta$ and $\theta_{\lambda h}$ so that $g=(U ,0)$, $\lambda h =(U_{\theta_{\lambda h}},0)$ where 
\begin{align*}
U=\begin{pmatrix}
\cos(\theta)&&-\sin(\theta)\\
\sin(\theta)&&\cos(\theta)
\end{pmatrix}, \quad
U_{\theta_{\lambda h}}=
\begin{pmatrix}
\cos(\theta_{\lambda h})&&-\sin(\theta_{\lambda h})\\
\sin(\theta_{\lambda h})&&\cos(\theta_{\lambda h})
\end{pmatrix},
\end{align*}
and note that $(g^{-1}\lambda h)(x)=U_{\theta_{g^{-1}\lambda h}}x$ where 
$$
U_{\theta_{g^{-1}\lambda h}}=\begin{pmatrix}
\cos(\theta_{\lambda h}-\theta)&&-\sin(\theta_{\lambda h}-\theta)\\
\sin(\theta_{\lambda h}-\theta)&&\cos(\theta_{\lambda h}-\theta)
\end{pmatrix}.
$$
Thus, Lemma~\ref{lem: ell infinity diff for different group members}, applied with $g_1=g^{-1}\lambda h$, $g_2=\lambda h$, implies
$$
\sup_{h\in H}\|\psi_{g^{-1}\lambda h}-\psi_{\lambda h}\|_{\mathbf{L}^2(\mathbb{R}^2)}\leq 2 \pi^{1/2}R^2\|\psi\|_L \sin\left(\frac{|\theta|}{2}\right),
$$
and so the result follows from Lemma~\ref{lem: continuity simple}.
\end{proof}

\section{Proofs of Propositions~\ref{prop: Psi Equivariant} and~\ref{prop: upsilon equivariance}}\label{app: proofs of standard props}

Propositions~\ref{prop: Psi Equivariant} and~\ref{prop: upsilon equivariance} are well known in abstract harmonic analysis, but may be new to readers who are primarily interested in neural networks. Therefore, we provide proofs in order to keep our work self-contained.

\subsection{Proof of Proposition~\ref{prop: Psi Equivariant}}
\label{app: proof of equivariance of Xi}
\begin{proof}
Let $g,\widetilde{g}\in G$ be arbitrary. It suffices to prove
\begin{align}
    \langle (g\cdot f), \psi_{\widetilde{g}}\rangle_{\bL^2(\mathbb{R}^d)} 
&=\int_{\mathbb{R}^d}f(y)\psi_{g^{-1}\widetilde{g}}(y) \,dy. \label{eqn: equiv lemma key inequality}
\end{align}
 For then, by the definitions in~\eqref{eqn:quasireg}, \eqref{eqn: outer action}, and~\eqref{eqn:Psi}, we will have 
\[
(g\cdot \Psi(f))(\widetilde{g}) = \Psi(f)(g^{-1}\widetilde{g})
=\int_{\mathbb{R}^d}f(y)\psi_{g^{-1}\widetilde{g}}(y) \,dy
= \langle (g\cdot f), \psi_{\widetilde{g}}\rangle_{\bL^2(\mathbb{R}^d)}=\Psi(g\cdot f)(\widetilde{g})
\]
as desired.
Now, to prove~\eqref{eqn: equiv lemma key inequality}, we make the change of variables $y=g^{-1}(x)$ to obtain
\begin{align}
    \langle (g\cdot f), \psi_{\widetilde{g}}\rangle_{\bL^2(\mathbb{R}^d)} &= \int_{\mathbb{R}^d} (g\cdot f)(x)(\widetilde{g}\cdot \psi)(x)\,dx \nonumber\\
&=\int_{\mathbb{R}^d} \frac{1}{|\det(Jg)|^{1/2}} f(g^{-1}(x))\frac{1}{|\det(J\widetilde{g})|^{1/2}}\psi(\widetilde{g}^{-1}(x))\,dx\nonumber\\
&=\frac{|\det(Jg)|^{1/2}}{|\det(J\widetilde{g})|^{1/2}}\int_{\mathbb{R}^d}f(y)\psi(\widetilde{g}^{-1}(
g(y)
))\,dy\nonumber\\
&=\frac{|\det(Jg)|^{1/2}}{|\det(J\widetilde{g})|^{1/2}}\int_{\mathbb{R}^d}f(y)\psi((g^{-1}\circ \widetilde{g})^{-1}(y))\,dy\nonumber\\
&=\int_{\mathbb{R}^d}f(y)\psi_{g^{-1}\widetilde{g}}(y) \,dy, \nonumber 
\end{align}
where, in the final two equalities, we use the facts that  
\[
\widetilde{g}^{-1}\circ g=(g^{-1}\circ \widetilde{g})^{-1}, 
\qquad
\frac{|\det(Jg)|^{1/2}}{|\det(J\widetilde{g})|^{1/2}}=\frac{1}{|\det(J(g^{-1}\circ \widetilde{g}))|^{1/2}}.
\qedhere
\]
\end{proof}

\subsection{Proof of Proposition~\ref{prop: upsilon equivariance}}
\label{app: proof of upsilon equivarariance} 

\begin{proof}
Let $\upsilon\in \bL^1(\Lambda)$, $\lambda \in \Lambda$, and $F\in\bC_b(\Lambda)$. 
Then for  $\widetilde{\lambda}\in\Lambda$ and $\mu_\Lambda$ the Haar measure of $\Lambda$, we have
\begin{align}
\Upsilon({\lambda}\cdot F)(\widetilde\lambda)&=\int_\Lambda ({\lambda}\cdot F)(\kappa)\upsilon(\widetilde\lambda^{-1}\kappa)\,d\mu_{\Lambda}(\kappa)\nonumber\\
&=\int_\Lambda F(\lambda^{-1}\kappa)\upsilon(\widetilde\lambda^{-1}\kappa)\,d\mu_{\Lambda}(\kappa)\nonumber\\
&=\int_\Lambda F(\tilde{\kappa})\upsilon(\widetilde\lambda^{-1}\lambda\tilde{\kappa} )\,d\mu_{\Lambda}(\tilde{\kappa})\label{eqn: h tilde sub}\\
&=\int_\Lambda F(\tilde{\kappa})\upsilon(({\lambda}^{-1}\widetilde\lambda)^{-1}\tilde{\kappa} )\,d\mu_{\Lambda}(\tilde{\kappa})\nonumber\\
&=\Upsilon(F)(\lambda^{-1}\widetilde\lambda)\nonumber\\
&=(\lambda \cdot \Upsilon(F))(\widetilde\lambda)\nonumber
\end{align}
where in~\eqref{eqn: h tilde sub} we substituted $\tilde{\kappa}=\lambda^{-1}\kappa$.
\end{proof}

\section{Implementation Details}\label{app: implementation}

We implement the \ournet{} architecture in PyTorch \cite{pytorch} via a combination of custom layers and non-standard uses of standard PyTorch operations. All training and inference were performed on a single NVIDIA GH200 Grace Hopper Superchip.

To implement the operator $\Psi$, we use two separate layers. To justify this, we observe that for $g=(A,b)\in G$, we can factor $g=g_1\circ g_2$, where $g_1=A $ and $g_2
=(I,A^{-1}b)$. Thus, we can rewrite Equation~\ref{eqn:Psi} as
\[
\Psi(f)(g) = \langle f, g \cdot \psi \rangle_{\bL^2(\bR^d)} = \langle g_1^{-1} \cdot f, g_2 \cdot \psi \rangle_{\bL^2(\bR^d)}.
\]
We then compute a digital approximation of $g_1^{-1} \cdot f$ (using resampling and cropping) for all $g_1$ as the first layer, then apply the inner product as a convolution operation. The first layer is a custom layer making use of PyTorch's deterministic transform capabilities. For the cases in this paper, which all utilize rotation matrices for $A$, this first layer amounts to creating rotated copies of the input $f$ and ``stacking'' them along a new dimension. 
The second layer implements the inner product and the translation component $g_2$ using a 3D convolution with a $k\times k\times1$ kernel.
This has the effect of convolving (in the 2D sense) a $k\times k$ kernel across each rotated image in the stack while preserving the stack structure.

We implement $\Upsilon$ using a $k\times k\times1$ convolutional kernel, which applies the same spatial convolution independently to each orientation slice. 
In the corresponding continuous model with a finite rotation group, this is a special case of~\eqref{eqn: upsilon} in which the kernel is supported at the identity orientation. 
This restriction avoids mixing orientation slices and permits implementation using a standard 3D convolution layer.
 
To implement $\Xi$, we first note that when the only nontrivial component of $H$ is in the $\bR^d$ component, $\Xi$ can be implemented as standard spatial pooling, which in PyTorch means using a 3D max pooling operation with a $k \times k \times 1$ filter. On the other hand, when the only nontrivial component of $H$ is in the $\GL(\mathbb{R}^d)$ component, the foundation of this layer once again utilizes the standard PyTorch 3D max pooling operation (this time with a $1 \times 1 \times |H|$ filter), but before we can apply that pooling operation, we must ensure that the pixels we are pooling together are aligned.  

To understand the alignment, recall from the caption of Figure~\ref{fig:\ournet{} example} that the
counter-rotated parameterization of $F \in \bC_b(G)$ is given by
$\widetilde F(A,b)=F(A,Ab)$.
For rotational pooling, write
$H=\{(B,0):B\in H_D\}$, where $H_D$ is the set of rotation
representatives. Since
$(A,Ab)(B,0)=(AB,Ab)$, the definition of $\Xi$ gives
\[
(\Xi F)(A,Ab)
=
\sup_{B\in H_D} F(AB,Ab)
=
\sup_{B\in H_D} \widetilde F(AB,B^{-1}b).
\]
For a rotation $B\in H_D$, the function
$b\mapsto\widetilde F(AB,B^{-1}b)$ is obtained by rotating the
slice $\widetilde F(AB,\cdot)$ by $B$.
Thus, within each pooling batch, we apply the relative rotation
$B$ to the slice indexed by $AB$, aligning it with the coordinate
frame of the retained orientation $A$.
We then apply standard 3D max pooling across the aligned slices.
See Figure~\ref{fig:\ournet{} example} for an illustration.
A similar procedure applies to more general (not strictly rotational) $\GL(\mathbb{R}^d)$ transformations as well.

While in theory, we could use a $k_1 \times k_1 \times k_2$ pooling filter to pool both in the $\GL(\mathbb{R}^d)$ and $\bR^{d}$ components simultaneously, in practice we apply two sequential pooling operations of the types described above, one with a $k_1 \times k_1 \times 1$ filter and one with a $1 \times 1 \times k_2$ filter. This aligns with standards in the literature (e.g.,~\cite{cohen2016group,romero2020wavelet,dieleman2016exploiting,fasel2006rotation,sifre2013rotation,ver2023using}).

\subsection{Implementation Details for Section~\ref{sec: continuity of pooling}}
\label{subsection.section 4.1}

The size of the input image is $128 \times 128$, and the 3D convolution layer (see above) uses a $128 \times 128 \times 1$ filter with ``valid only'' padding style so that there is no translational component to $\Psi$, only rotational. 
Our filters are not trained, and as such have no optimizer or loss function. Instead, the filter bank for the 3D convolution layer is randomly initialized using PyTorch's default weight initialization scheme with a manual seed to ensure replicability.

\subsection{Implementation Details for Section~\ref{sec: decreasing training set}}
\label{subsec.C2}
We create and train three networks, a \ournet{},  a conventional CNN without data augmentation, and a conventional CNN with data augmentation (architectures  are shown in Figure~\ref{fig: synthetic model architectures}). The augmented CNN uses rotational data augmentation via PyTorch's random rotation module, with each raw input being augmented by 2 additional rotated copies, with random rotations of up to $60^\circ$ clockwise or counterclockwise.
 
The following hyperparameters are shared across all three models: leaky ReLU activation functions with a slope of $0.2$, a batch size of $256$ (of unaugmented initial samples, in the case of the augmented CNN), an Adam optimizer with an initial learning rate of $0.001$, early stopping with a patience of $20$ epochs and a required improvement of $0.0001$, a maximum number of training epochs of 1000 epochs (which is never achieved due to early stopping occurring well before that point). Due to the larger memory footprint of its feature activations, the \ournet{} uses gradient accumulation to achieve the same effective batch size as the CNNs --- each batch is only 64 samples, but gradients are accumulated across 4 batches before each optimizer step is taken. There is one difference to note when using gradient accumulation: the batch normalization layers in the model apply their normalization per smaller sub-batch rather than per optimization step. These hyperparameters are heuristically chosen based on performance on the validation set, but no significant hyperparameter search is undertaken. Both CNN models have 1,123,792 trainable parameters, while the \ournet{} model has 1,075,672 trainable parameters.

Additionally, all models use a continuous ranked probability score (CRPS) loss \cite{matheson1976scoring, haynes2023creating}, and we choose a 100-member ensemble representation to implement the loss.
For evaluation, we average the 100~ensemble outputs and predict class~1 when this mean exceeds $0.5$, and class~0 otherwise.
The reported classification metrics are computed from these predicted labels.
While the CRPS loss is somewhat unusual for binary classification, it is used to maintain consistency with the real-world tropical cyclone example presented in Section~\ref{sec: experiments real world}, which in turn uses CRPS to match the hyperparameter choices made in \cite{lagerquist2025identifying}.

\begin{figure}[htp]
    \centering
    \begin{subfigure}[b]{1.0\textwidth}
        \centering
        \includegraphics[width=1.0\linewidth]{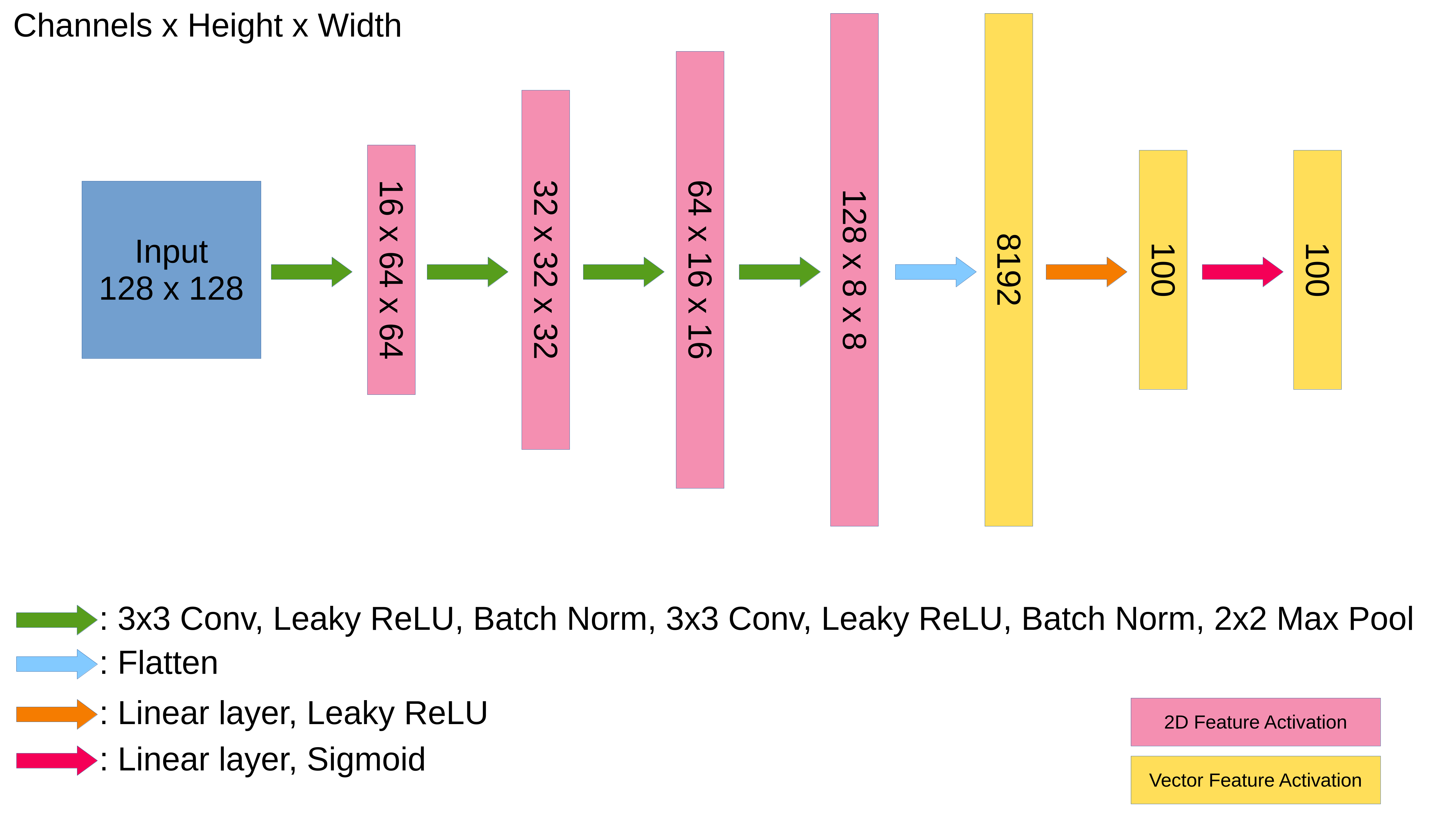}
        \caption{CNN}
    \end{subfigure}
    \hfill
    \vspace{5mm}
    \begin{subfigure}[b]{1.0\textwidth}
        \centering
        \includegraphics[width=1.0\linewidth]{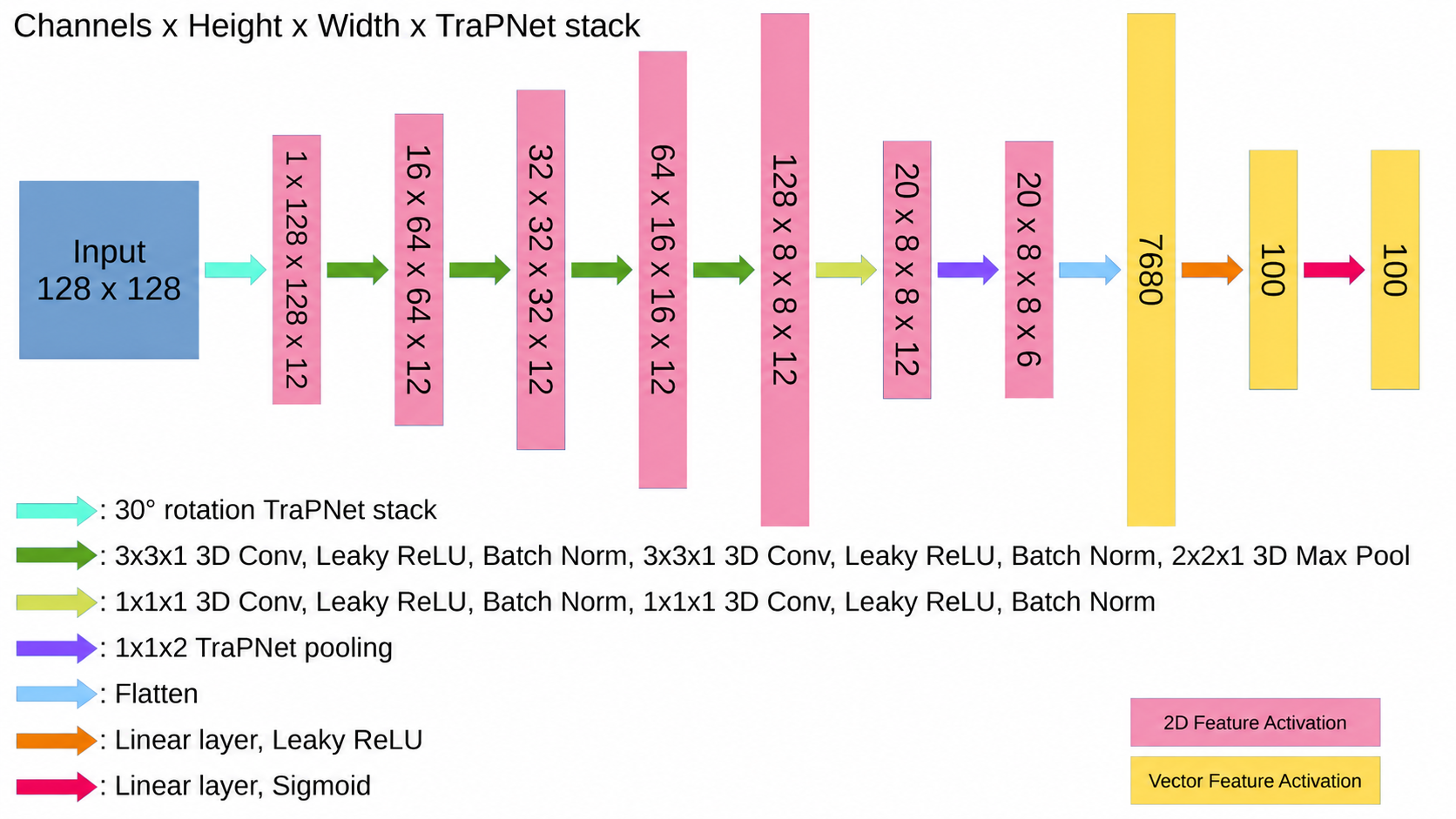}
        \caption{\ournet{}}
    \end{subfigure}
    \caption{Model architectures for the baseline CNN and the \ournet{} used in Section~\ref{sec: decreasing training set}.}
    \label{fig: synthetic model architectures}
\end{figure}

\subsection{Implementation Details for the Real-World Tropical Cyclone Example}
\label{subsec.real world experiment}

The hyperparameter choices used in Section~\ref{sec: experiments real world} are mostly the same as in Section~\ref{sec: decreasing training set}. Once again, for both models we use the CRPS loss function, leaky ReLU activations with a slope of 0.2, the Adam optimizer with a learning rate of $0.001$, early stopping with a patience of $20$ epochs and a required improvement of $0.0001$, and a maximum of 1000 epochs (which is again never achieved). The main hyperparameter differences come from the different dataset. Because we use the weighted random sampler, we must choose the number of batches per epoch, which is set to 64, and due to the larger images and larger networks, we use a smaller batch size of $96$ images per batch (original images, prior to any data augmentation). Once again, we use gradient accumulation for the \ournet{}, so the true batch size is set to 12 with gradients being accumulated across 8 of those batches to achieve an effective batch size of 96. Because we use early stopping, we apply the same weighted sampler to the validation set at the end of each epoch so that the early stopping model selection is not overly influenced by the negative class. The CRPS loss function, Adam optimizer, learning rate, and early stopping are all chosen to match with the hyperparameters used in \cite{lagerquist2025identifying}.

For this task, we also follow \cite{lagerquist2025identifying} in applying two forms of data augmentation to both models, including the \ournet{}. Each input image is augmented by three randomly translated copies with a maximum shift of $5\%$ of the image width/height in the appropriate direction, and with two randomly noised copies with Gaussian white noise independently sampled for each pixel with a standard deviation of 0.5 (which is applied to the standardized and normalized input images). The ``augmented'' CNN also includes 2 randomly rotated copies, with a maximum rotation of $50^\circ$ clockwise or counterclockwise.

We again use a 100-member ensemble representation to implement the CRPS loss. 
Classification metrics use the decision rule described in Appendix~\ref{subsec.C2}: class~1 is predicted when the ensemble mean exceeds $0.5$, and class~0 otherwise.
We use CRPS both to match the choices in~\cite{lagerquist2025identifying} and because, in our experiments, it produces better and more reliable performance than binary cross-entropy.

Architecture diagrams for both the CNN and \ournet{} are shown in Figure~\ref{fig: real world architectures}. The baseline augmented CNN has 22,840,918 trainable parameters, while the \ournet{} has 22,575,286 trainable parameters.

\begin{figure}[htp]
    \centering
    \begin{subfigure}[b]{1.0\textwidth}
        \centering
        \includegraphics[width=1.0\linewidth]{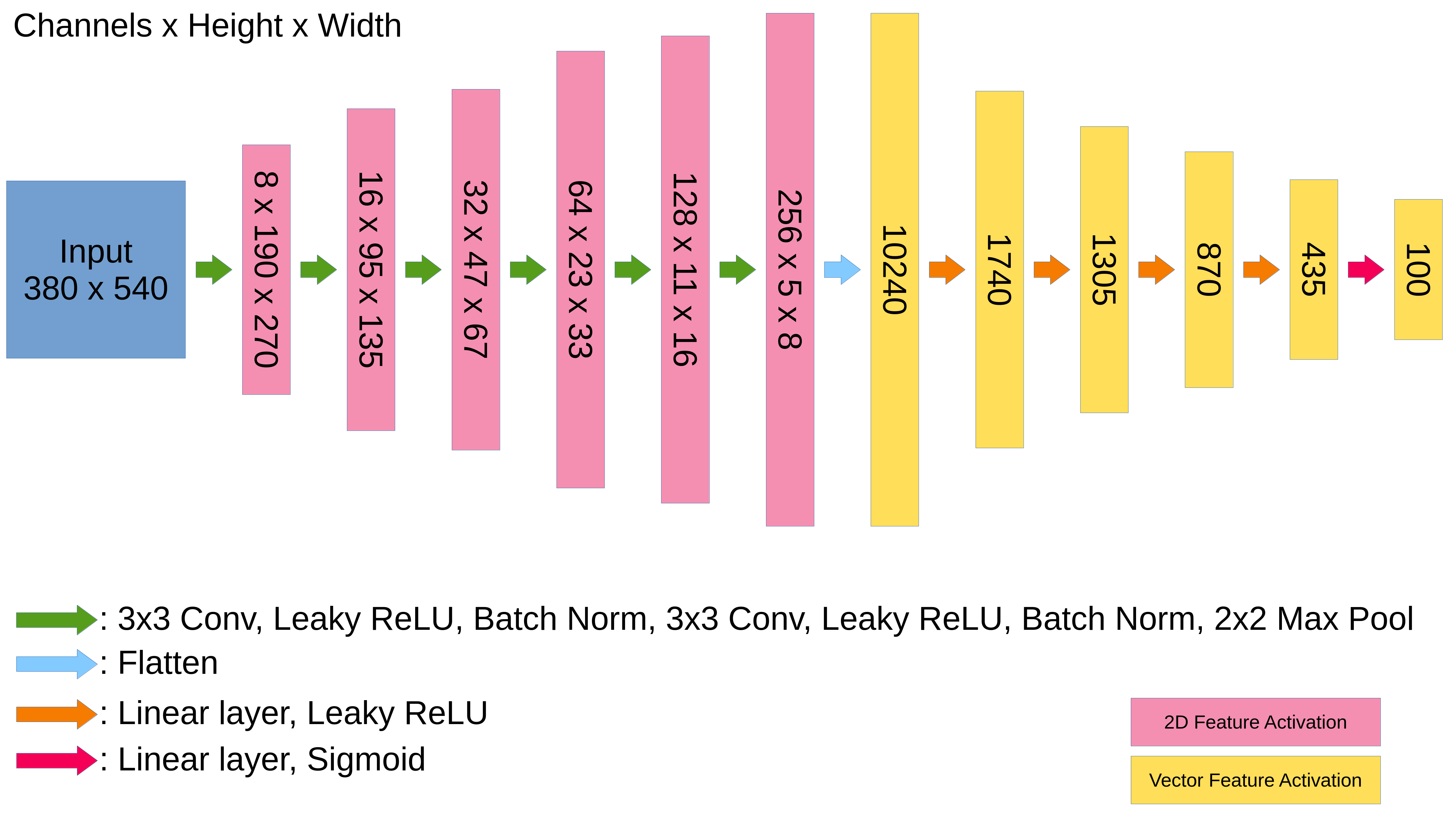}
        \caption{CNN}
    \end{subfigure}
    \hfill
    \vspace{5mm}
    \begin{subfigure}[b]{1.0\textwidth}
        \centering
        \includegraphics[width=1.0\linewidth]{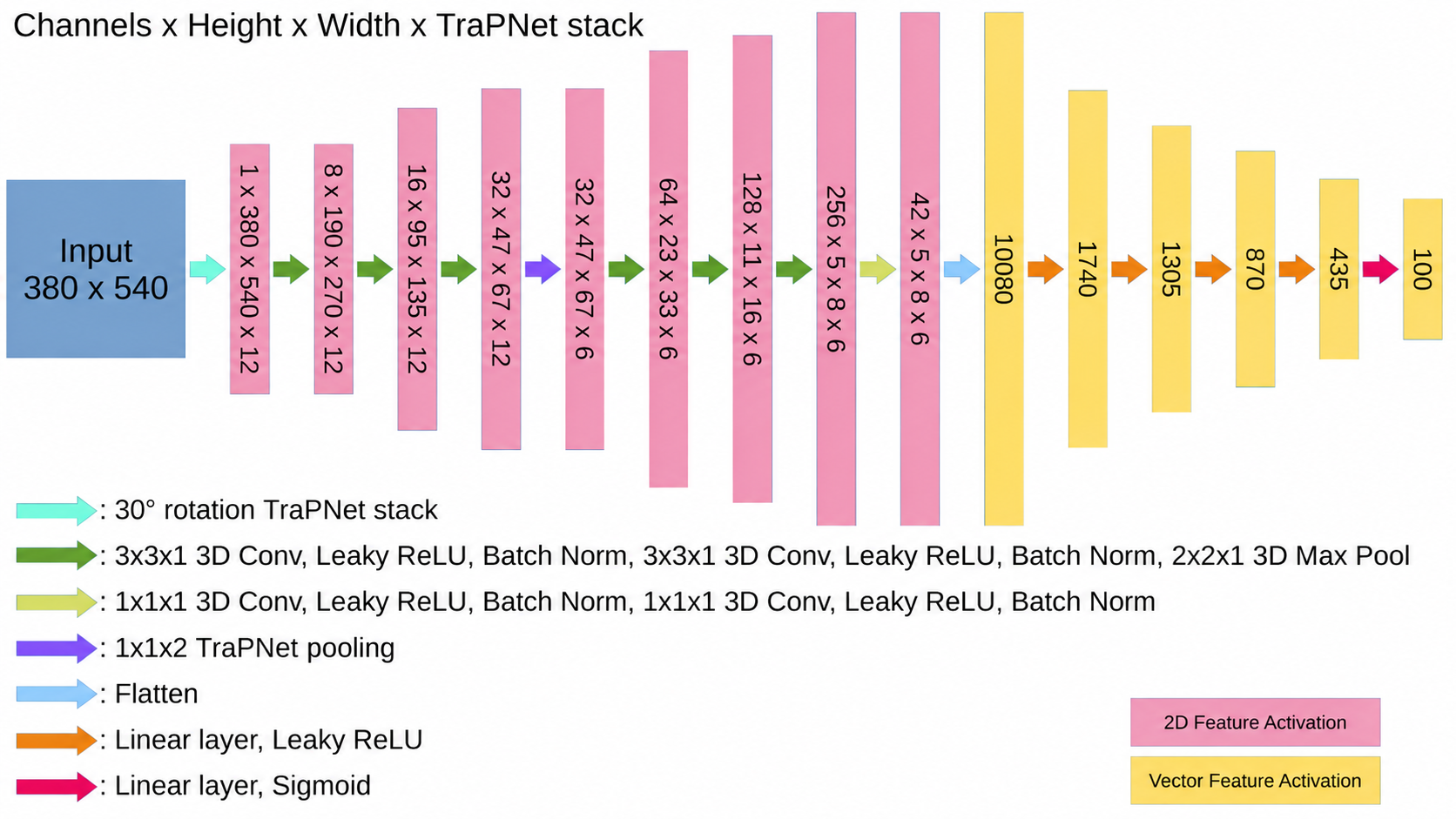}
        \caption{\ournet{}}
    \end{subfigure}
    \caption{The architectures for the baseline CNN and the \ournet{} used in Section~\ref{sec: experiments real world}.}
    \label{fig: real world architectures}
\end{figure}

\end{document}